\documentclass[10pt,a4paper]{article}
\usepackage{jheppub}
\usepackage{etoolbox}
\usepackage[utf8]{inputenc}
\usepackage[english]{babel}
\usepackage{amsmath}
\usepackage{amsfonts}
\usepackage{amssymb}
\usepackage{xcolor}
\usepackage{mathrsfs}
\usepackage{amsthm}

\newcommand{\be}{\begin{equation}}
\newcommand{\ee}{\end{equation}}

\theoremstyle{plain}
\newtheorem{Remark}{Remark}
\newtheorem{Theorem}{Theorem}
\newtheorem{Lemma}{Lemma}
\newtheorem{Proposition}{Proposition}
\newtheorem{Corollary}{Corollary}
\newtheorem{Definition}{Definition}
\graphicspath{{./Figures/}}

\patchcmd{\maketitle}{\@fpheader}{}{}{}

\title{Dreaming neural networks:  forgetting spurious memories and reinforcing pure ones.}
\author[a,b,c]{Alberto Fachechi}
\author[d,e]{Elena Agliari}
\author[a,b,c]{Adriano Barra}

\affiliation[a]{Dipartimento di Matematica e Fisica Ennio De Giorgi, Universit\`a del Salento, Italy}
\affiliation[b]{GNFM-INdAM Sezione di Lecce, Italy}
\affiliation[c]{INFN, Istituto Nazionale di Fisica Nucleare, Sezione di Lecce, Italy}
\affiliation[d]{Dipartimento di Matematica, Sapienza Universit\`a di Roma, Italy}
\affiliation[e]{GNFM-INdAM Sezione di Roma, Italy}
\emailAdd{alberto.fachechi@le.infn.it}
\emailAdd{elena.agliari@uniroma1.it}
\emailAdd{adriano.barra@unisalento.it}
\abstract{The standard Hopfield model for associative neural networks accounts for biological Hebbian learning and acts as the {\em harmonic oscillator} for pattern recognition, however its maximal storage capacity  is $\alpha \sim 0.14$, far from the theoretical bound for symmetric networks, i.e. $\alpha =1$. \newline Inspired by sleeping and dreaming mechanisms in mammal brains, we propose an extension of this model displaying the standard on-line (awake) learning mechanism (that allows the storage of external information in terms of patterns) and an off-line (sleep) unlearning$\&$consolidating mechanism (that allows spurious-pattern removal and pure-pattern reinforcement): this obtained {\em daily prescription} is able to saturate the theoretical bound $\alpha=1$, remaining also extremely robust against thermal noise.
\newline
Both neural and synaptic features are analyzed both analytically and numerically. In particular, beyond obtaining a phase diagram for neural dynamics, we focus on synaptic plasticity and we give explicit prescriptions on the temporal evolution of the synaptic matrix. We analytically prove that our algorithm makes the Hebbian kernel converge with high probability to the projection matrix built over the pure stored patterns. Furthermore, we obtain a sharp and explicit estimate for the ``sleep rate'' in order to ensure such a convergence.
\newline
Finally, we run extensive numerical simulations (mainly Monte Carlo sampling) to check the approximations underlying the analytical investigations (e.g., we developed the whole theory at the so called {\em replica-symmetric} level, as standard in the Amit-Gutfreund-Sompolinsky reference framework) and possible finite-size effects, finding overall full agreement with the theory.}
\begin{document}
\maketitle

\section{Introduction: the starting points}

An \emph{intelligent} machine must be able to \emph{learn} new patterns of information and to \emph{retrieve} previously learnt ones as a response to external stimuli: these two intimately related concepts are the main aspects of cognition in Artificial Intelligence (AI). More sophisticated machines also exhibit the ability to \emph{reinforce} relevant memories (e.g. pure states) and to \emph{remove} irrelevant ones (e.g. mixture states), allowing a smarter storage of information. In this work, keeping the paradigmatic Hopfield model as the awake reference, we equip it with reinforcement and remotion features (able to work simultaneously,  during the network {\em sleep}, as inspired by real sleeping and dreaming mechanisms in mammal brains:
oversimplifying, a sleeping session can be split in two different modes: {\em rapid eye movement sleep} (REM sleep) and {\em slow wave sleep} (SW sleep); the former yields to erasure of unnecessary memories, the latter to consolidation of the important ones \cite{onde,Diekerlmann,Rash,unlearning4}.
Usually these two stages of sleep alternate during the night and, of relevance for synaptic homeostasis, the former is particular important in order to globally reduce synaptic strength (and its relative consumption of energy and tissue, an idea in agreement with the original Parisi proposal on forgetting neural networks \cite{Giorgio,Enzo}), while the latter is more dedicated to consolidation of relevant memories through some sort of off-line reinforcement learning \cite{amigdala1,amigdala2}.
\newline
In the Literature on Artificial Intelligence, reinforcement (of pure states) and remotion (of spurious states) are typically addressed separately (see {\it e.g.} \cite{HopfieldUnlearning,VanHemmen} for the former and \cite{RL1,RL2} for the latter). Here, instead, we propose a unified framework for synaptic plasticity where simultaneously reinforcement \emph{and} remotion take place. As we will see, the combined effect of these mechanisms determines a larger retrieval region, where retrieval is stable against both the fast and the slow noise. To our knowledge, the resulting associative network outperforms other models (with symmetric interactions) appeared in the Literature. In the remainder of this Section, we provide a short description of the state of the art focusing on those aspects that are mostly related to our work.
\par
Since the seminal work by John J. Hopfield in the eighties \cite{Hopfield}, associative neural networks have become the standard model to capture collective capabilities
spontaneously shown by networks of interacting neurons.
In a nutshell, a Hopfield network is made of $N$ units mimicking binary neurons, whose state (spiking/quiescent) is described by an Ising spin ($\sigma =\pm 1$). Units interact pairwise through weighted links mimicking synaptic connections, whose magnitude is defined according to Hebb's rule for learning, namely, given $P$ patterns of information $\{\xi^{\mu}\}_{\mu=1,...,P}$ of length $N$,
the coupling $J_{ij}$ between the neuron $i$ and the neuron $j$ reads
\be\label{eq:hebb}
J_{ij} \equiv \frac{1}{N}\sum_{\mu=1}^{P}\xi_i^{\mu}\xi_j^{\mu}, ~~ i,j=1,...,N.
\ee
Typically, one takes Boolean patterns with entries identically and independently drawn with equal probability, {\it i.e.}  $P(\xi_i^{\mu} = +1 ) = P(\xi_i^{\mu} = -1 ) = 1/2$.
Moreover, the set of patterns is taken as static\footnote{The time scale for neuronal dynamics is much shorter than the time scale for synaptic (and therefore pattern) dynamics, in such a way that, when focusing on retrieval tasks one can take synapsis as static, see {\it e.g.} \cite{Amit,Coolen}.} and are thus called {\em quenched}. In order to assess the retrieval of the $\mu^{th}$ pattern,
one introduces the so-called Mattis overlaps
\begin{equation}\label{eq:mattis}
m_{\mu} \equiv \frac{1}{N} \sum_{i=1}^{N}\xi_i^{\mu}\sigma_i, ~~ \mu=1,...,P,
\end{equation}
in such a way that, when the neuronal configuration $\{\sigma_i\}_{i=1,...,N}$ is aligned with $\xi^{\mu}$, then $m_{\mu}=1$; this configuration is interpreted as the retrieval of the pattern $\xi^{\mu}$. The Hopfield model is formally described by a cost-function (or {\em energy}, or {\em Hamiltonian}, to keep a physical jargon) $H_{N,P}(\sigma|\xi)$ defined as
\be\label{eq:hopfield-Hamiltonian}
H_{N,P}(\sigma|\xi) \equiv  -\sum_{i<j}^{N,N}J_{ij}\sigma_i \sigma_j \sim -\frac{1}{2N}\sum_{i,j=1}^{N,N} \sum_{\mu=1}^{P} \xi_i^{\mu}\xi_j^{\mu} \sigma_i\sigma_j = - \frac{N}{2}\sum_{\mu=1}^{P} m_{\mu}^2.
\ee
Here, the first sum runs over all possible pairs of neurons, while in the second passage we neglected $O(N^{-1})$ terms and implemented the Hebb coupling (\ref{eq:hebb}), and lastly we used the definition (\ref{eq:mattis}). Once that the cost-function $H_{N,P}(\sigma|\xi)$ is given, exploiting the mean-field nature of the model, a neural dynamics can be easily constructed \cite{Amit,Coolen}. To this goal, we introduce the \emph{fast noise} ({\it i.e.} standard white noise, or {\em temperature} in physical jargon)
whose magnitude is tuned by a parameter $T \equiv 1/\beta$ (such that as $T \to \infty$ the neural update is entirely random, while as $T \to 0$ it reduces to a deterministic evolution \cite{Amit,Coolen}) and the internal field $h_i$ acting on the $i^{th}$ neuron. In this way, the Hopfield cost-function (\ref{eq:hopfield-Hamiltonian}) can be written as
\begin{eqnarray}\label{eq:withfield}
H_{N,P}(\sigma|\xi) &=& - \sum_{i=1}^{N} h_i \sigma_i,  \ \ \ \ \ h_i = \frac{1}{2N}\sum_{j=1}^{N} \sum_{\mu=1}^{P} \xi_i^{\mu}\xi_j^{\mu} \sigma_j,
\end{eqnarray}
and the stochastic neural update rule can be written as
\begin{equation}
\label{eq:evolution-rule}
P\left(\sigma_1(\tau+1),...,\sigma_N(\tau+1)\right) = \prod_{i=1}^{N}\left \{ \frac12 \left[ 1 + \sigma_i(\tau) \cdot \tanh(\beta h_i(\tau)) \right] \right \},
\end{equation}
where the parameter $\tau$ identifies a suitable neural-update timescale. This dynamics ensures that detailed balance holds and the neural configuration eventually converges to the Boltzmann-Gibbs distribution associated to the Hamiltonian (\ref{eq:hopfield-Hamiltonian}), the latter playing as a Lyapunov function at $T=0$.
\newline
This system can be addressed via sophisticated techniques stemming from the statistical mechanics of disordered systems, as pioneered by Amit-Gutfreund-Sompolinksy (AGS) \cite{AGS1,AGS2}. Before moving to that, we sketch a heuristic argument, due to Hopfield and Tank \cite{HopfieldTank}, to see the retrieval capabilities of the network. Since patterns are randomly generated, for an arbitrary vector state $\{ \sigma_i \}_{i=1,...,N}$ the related Mattis magnetizations would vanish as $O(N^{-1/2}$) and the corresponding contribution to the cost-function (\ref{eq:hopfield-Hamiltonian}) is negligible. On the other hand, for a vector state that is (partially) aligned with a given pattern, the contribution to the energy would be $O(N)$, and it therefore occurs to be a convenient (stable) state for the system.
This suggests that the model displays energy minima at each of the assigned memories. In order to strengthen this picture, statistical mechanics definitions and tools are now essential.
\par
The (intensive) free energy associated to the cost-function (\ref{eq:hopfield-Hamiltonian})
is defined as
\be\label{eq:free-energy-def}
A_{N,P}(\beta) \equiv - \frac{1}{\beta N}\mathbb{E}\ln Z_{N,P}(\beta|\xi),
\ee
where $Z_{N,P}(\beta|\xi)\equiv \sum_{\{ \sigma \}}\exp[-\beta H_{N,P}(\sigma|\xi)]$ is called the {\em partition function}  and $\mathbb{E}$ denotes the average over the patterns, also termed {\em slow noise}.
In the following, we will mainly focus on the thermodynamic limit of the free energy, namely $A(\alpha,\beta) \equiv \lim_{N\rightarrow \infty} A_{N,P}(\beta)$, where $\alpha \equiv \lim_{N\to\infty}P/N$ is referred to as the \emph{load} (or {\it storage capacity}) of the system.
\newline
In a statistical-mechanical approach, one aims to express the free energy explicitly in terms of the {\em order parameters}, namely simple functions of the state of the system under study giving informations about the behavior of the system itself. In this context, the $P$ Mattis overlaps $\mathbf{m} = (m_1, ..., m_P)$ work as order parameters since, according to their value, one can infer whether the network is retrieving ($\exists \mu\, | \, m_{\mu} \neq 0$) or not ($m_{\mu} = 0, \forall \mu$). Further, the so-called Edward-Anderson overlap $q_{ab} \equiv N^{-1} \sum_{i}^{N}\sigma_i^{a}\sigma_i^b$ is likewise useful as it detects the {\em spin-glass} regime (namely a region where ``structured disorder'' prevails, as will be explained in more details later) \cite{Amit,Coolen}. In the replica symmetric approximation provided by AGS theory, the free energy of the Hopfield model expressed in terms of the Mattis and Edward-Anderson overlaps ($q_{ab} \equiv q$ for simplicity) reads as
\begin{eqnarray}\label{eq:hopfield-free-energy}
A(\alpha,\beta) &=& \frac{1}{2}\bold{m}^2 -\frac{1}{\beta}\int_{-\infty}^{+\infty}d \mu(x)\left\langle\ln \cosh \left \{ \beta \left( \bold{m}\cdot \boldsymbol{\xi} + x \frac{\sqrt{\alpha q}}{[1-\beta(1-q)]^2} \right)\right\} \right\rangle_{\boldsymbol{\xi}}\\ \nonumber
&+& \frac{\alpha}{2\beta}\ln\left[{1-\beta(1-q)} \right] - \frac{\alpha }{2} \frac{q}{1-\beta(1-q)}+\frac{\alpha}{2}\frac{q}{[1-\beta(1-q)]^2}.
\end{eqnarray}
\begin{figure*}[t]
	\centering
	\begin{minipage}[c]{.7\textwidth}
		\centering
		\includegraphics[width=\textwidth]{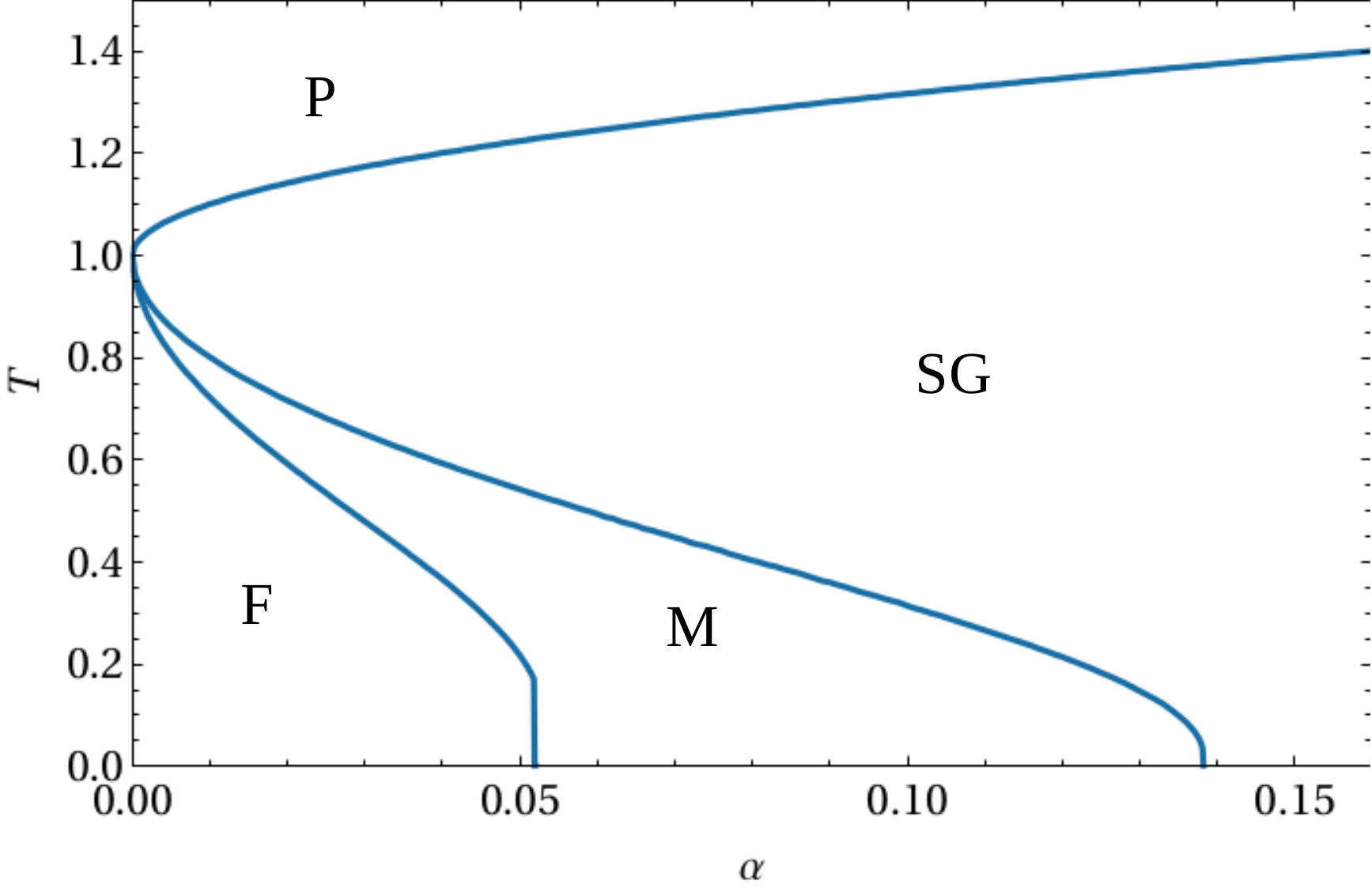}
	\end{minipage}
\caption{{\bfseries Phase diagram of the Hopfield network.} The phase diagram lives in the $(\alpha,\beta)$ plane. In the upper region (P) the network behaves randomly while in the top-right region (SG) it is frozen in a spin-glass phase. The working regions are solely the two down-left where patterns are global minima of the free energy (F) or relative minima of the free  energy (M). %
} \label{fig:equivalenza}
\end{figure*}
Recalling that, in its original interpretation in Thermodynamics \cite{Amit}, the free energy equals the difference between the energy  ({\it i.e.} the expectation value of the cost-function) and the entropy (related to the probability of observing a configuration $\{ \sigma_i \}_{i=1,...,N}$), the extremization of the free energy over the order parameters ensures simultaneously the minimum energy and the maximum entropy principles. Thus, the exploration of the free-energy landscape (as the noise level $\beta$ and the load $\alpha$ are tuned) allows us to inspect the system thermalization and equilibria.
Remarkably, as pointed out by Jaynes \cite{Jaynes}, this route has a clear meaning also from a statistical inference perspective, much closer in spirit to Machine Learning: minimizing the free-energy equals searching for the minima of the cost-function under the constraint of Maximum Entropy (see also \cite{Bialek}).
\newline
As anticipated, the order parameters for the Hopfield model are the $P$ Mattis overlaps $m_{\mu}$ and the Edward-Anderson overlap $q$ and, by extremizing the free-energy (\ref{eq:hopfield-free-energy}) with respect to such variables, we get the following self-consistent equations:
\begin{eqnarray}\label{eq:self-m}
\frac{dA(\alpha,\beta)}{d \bold{m}} = 0 \Rightarrow \bold m &=& \int_{-\infty}^{+\infty}d \mu(x) \left \langle \boldsymbol{\xi} \tanh\left [ \beta \left ( \bold{m}\cdot \boldsymbol{\xi} + x \frac{\sqrt{\alpha q}}{1-\beta(1-q)} \right)  \right] \right \rangle_{\boldsymbol{\xi}}, \\ \label{eq:self-q}
\frac{dA(\alpha,\beta)}{d q} = 0 \Rightarrow q &=& \int_{-\infty}^{+\infty}d \mu(x) \left \langle  \tanh^2\left[ \beta \left( \bold{m}\cdot \boldsymbol{\xi} +  x \frac{\sqrt{\alpha q}}{1-\beta(1-q)} \right)  \right] \right \rangle_{\boldsymbol{\xi}},
\end{eqnarray}
where the bracket $\langle \cdot \rangle_{\boldsymbol{\xi}}$ means the average over the quenched patterns.
\newline
By studying the solutions of these equations, one can obtain a {\em phase diagram} for the Hopfield network, namely, in the $(\alpha, \beta)$ plane one can distinguish three phases characterized by different solutions of Eqs. (\ref{eq:self-m}, \ref{eq:self-q}) and qualitatively different behaviors of the system. In our opinion, this is the greatest reward by the statistical-mechanics approach: 
the concept of phase diagram allows researchers to predict the network response as a function of the tunable parameters, and this can be a fundamental information in the modern theoretical foundation of AI.
The phase diagram\footnote{Despite almost four decades has elapsed since Hopfield's seminal work, a rigorous control of the entire phase diagram is still beyond the current mathematical technologies as the low temperature analysis of such disordered systems is notoriously difficult \cite{Tala,Viktor}. In neural networks we usually rely on the so-called {\em replica symmetric approximation} for the description of the free-energy landscape, as originally outlined by AGS \cite{Amit,AGS1,AGS2}. More details on replica symmetry will be presented in Sec.~\ref{replica-simmetric-theory}.} for the Hopfield model is shown in Fig. \ref{fig:equivalenza}, and one can detect:
\begin{itemize}
\item \emph{Ergodic phase}: in the high-temperature limit, the fast noise in the system is too strong for the neurons to reciprocally feel each other, therefore the system behaves randomly and no emergent collective properties of neurons can be appreciated. This region is characterized by $\bold{m}=0$ and $q=0$.

\item \emph{Spin-glass phase}: in the high-load limit, the slow-noise is too large for the neurons to correctly handle the whole set of patterns, and, again, the system fails to retrieve information. This region is characterized by $\bold{m}=0$ but $q \neq 0$.

\item \emph{Retrieval phase}: when both fast and small noise are relatively small, the system behaves as an associative neural network  and neural collective capabilities spontaneously appear. The phase is  characterized by $\bold{m} \neq 0$ and $q \neq 0$.
    \newline
    This region can be further split in a pure retrieval region (where pure states are global minima) and in a mixed retrieval region (where pure states are local minima, yet their attraction basin is large enough for the system to end there if properly stimulated).
\end{itemize}
As is clear from the phase diagram of the Hopfield model, the Hebbian prescription (\ref{eq:hebb}) implies a limitation in the number of patterns that the network can correctly handle. The network can, at most, manage a number of patterns $P$ that grows linearly in the number of neurons $N$ available, {\it i.e.} $P = \alpha N$ \cite{AGS2}. Once a critical threshold $\alpha_c\sim 0.14$ is reached, the network experiences a plethora of unpleasant symptoms, ranging from the worst scenario (the abrupt transition to the spin-glass phase, sometimes called {\em blackout catastrophe}, affecting solely fully connected models), to milder confusional states where network's performances are sensibly reduced, if not lost at all.\footnote{It is worth pointing however that, despite the severe criticism - see {\it e.g.} \cite{French} - initially raised against the {\em connectionist perspective} this approach belongs to (as far as the Hopfield blackout scenario is concerned), in diluted models this abrupt transition gets smoothed and switches from first-order to second-order (in the Ehrenfest notation), thus cross-talks effects in more realistic networks certainly are present (and strong), but the discontinuous lost of the whole information is just a chimera of fully connected models \cite{Ton}.}
\newline
The reason underlying this impasse is that the free-energy landscape of the system is characterized by pure-state minima (corresponding to  pure pattern retrieval) but also by spurious-state minima (corresponding to mixtures of patterns that are interpreted as errors); as long as the network is fed by a linear increment (in the neural volume $N$) of patterns to be stored (i.e. $P \propto N$), there is an unavoidable, combinatorial (roughly exponential in $N$) proliferation of mixed patterns which may work as ``traps'' for the system state \cite{Amit,Coolen}.
\par
In the late eighties, Elisabeth Gardner found, by general arguments, that the maximal theoretical capacity for symmetric networks\footnote{The maximal theoretical capacity reaches $\alpha_c =2$ for asymmetric networks.} is $\alpha_c=1$ \cite{Gardner}, so the Hopfield model threshold $\alpha_c \sim 0.14$ has always been looked at as rather poor: indeed, along the decades, scientists tried to improve the maximal capacity by implementing some extensions and variations on theme ({\it e.g.} keeping the network out of equilibrium \cite{offequilibrium1,offequilibrium2}, allowing the network to process multiple tasks at once \cite{Agliari-Dantoni,Agliari-Isopi,Agliari-PRL1}).
In this context, particularly appealing works were inspired by Crick and Mitchinson's paper \cite{Crick}, where it was argued that the REM phase of sleep  in mammals may serve to delete all the (involuntarily stored) irrelevant information (in order to save memory and avoid overloading catastrophes). Further evidences toward this hypothesis were found both on the empirical \cite{onde,Rash,Diekerlmann,consolidation,unlearning4} and the theoretical \cite{DotsenkoDorotheyev,Horas,Albert,unlearning0,unlearning1,Plakhov,VanHemmen} level.
\par
A first crucial contribution to frame this idea in AI was achieved by Hopfield himself (with Feinstein and Palmer \cite{HopfieldUnlearning}), and can be summarized as follows:  the spin-glass transition occurring as $\alpha$ increases beyond $\alpha_c$ ultimately originates from the fact that the number of spurious states are exponentially more abundant than the number of pure states (regardless their depth in the free energy landscape), such that, making a quench from infinite to zero temperature, the system would get trapped with higher probability in one of these mixtures. By sampling a number of these final configurations, one can measure the average pairwise correlation $\langle \sigma_i \sigma_j \rangle_{\textrm{mix}}$, for any $i$ and $j$. Next, one updates the coupling matrix performing an {\it inverse} Hebbian rule so that these mixture states are effectively removed:
\be\label{eq:HopUn}
J_{ij} \to J_{ij} - \frac{\epsilon}{N}\langle \sigma_i \sigma_j \rangle_{\textrm{mix}} = \frac{1}{N}\sum_{\mu=1}^{P}\xi_i^{\mu}\xi_j^{\mu}- \frac{\epsilon}{N} \langle \sigma_i \sigma_j \rangle_{\textrm{mix}},
\ee
where $\epsilon$ is a tunable (but small) parameter called {\it unlearning strength}.
Such a procedure should be iterated in order to progressively clean the free-energy landscape from these traps.\footnote{We stress that the normalization factor $N^{-1}$ in front of the term $\langle \sigma_i \sigma_j \rangle_{mix}$ is appropriately chosen if the unlearning algorithm is iterated $O(N)$ times. We will deepen this point (the amplitude of the unlearning or consolidating rates) in  Section \ref{sezione-quattro} and in the Appendix \ref{app:convergence}.} The minus sign in (\ref{eq:HopUn}) is responsible for the so-called \emph{unlearning} process. A further motivation for the analogy with the REM phase is supplied by the fact that, during these REM phases, dreams are not entirely uncorrelated with the experiences we actually lived during the wakefulness state and there are (possibly weird) correlations between dreams and these experiences; a similar scenario happens in the artificial side as spurious states are just mixtures of patterns\footnote{The typical example is given by the symmetric mixtures of three patterns, that is $\sigma_i = \textrm{sign}(\xi_i^{1}+\xi_i^{2}+\xi_i^{3})$.} that unavoidably implies short-length correlations with the pure patterns.
In the same spirit as Hopfield's proposal, Plakhov and Semenov \cite{Plakhov} realized an unlearning algorithm by replacing the pure pairwise correlations between spins with correlations between inner fields, namely
\be\label{eq:unlearningpl}
J_{ij} \to J_{ij} - \epsilon\langle h_i h_j \rangle ,
\ee
where the average $\langle \cdot  \rangle$ is performed on a sample of randomly selected states in the configuration space with internal fields $h_i$. The main result is that, with a suitable choice of the unlearning strength, this algorithm is ensured to converge (up to scaling factors) to the projector (or pseudo-inverse) matrix
\begin{equation} \label{eq:J_K}
J_{ij}^{p} = \frac{1}{N}\sum_{\mu,\nu=1}^{P,P}\xi_i^{\mu}(C^{-1})_{\mu,\nu}\xi_j^{\nu},
\end{equation}
where
\begin{equation}
C_{\mu,\nu} = \frac{1}{N}\sum_{i=1}^{N}\xi_i^{\mu}\xi_i^{\nu},
\end{equation}
is the pattern correlation matrix. Notably, in this model, similar in spirit to the origina Kohonen idea \cite{kohonen} but closer in its statistical mechanical construction to the model introduced and studied by Kanter and Sompolinsky in \cite{KanterSompo}, the storage capacity reaches $\alpha_c=1$. A closely related model, studied by Dotskenko and coworkers \cite{DotsenkoDorotheyev,DotsenkoTirozzi}, is based on the following coupling matrix\footnote{While quite marginal in AI, it is still worth stressing that such a learning rule is no longer {\em local}, like the Hebbian prescription, in fact, the coupling between neurons $i$ and $j$ now depends on pattern entries related to all the neurons making up the system. In the biological world this point constitutes a modeling weakness, however Dotsenko and coworkers have shown how to bypass it in order to obtain roughly the same results \cite{DotsenkoDorotheyev}.
Another local algorithm able to converge to the projector matrix is the called Adeline learning rule (see \cite{Kinzel,VanHemmen} for an overview.)}
 \begin{equation} \label{eq:J_D}
J_{ij}({t}) = \frac{1}{N}\sum_{\mu,\nu=1}^{P,P}\xi_i^{\mu}(\mathbb{I} + t C)_{\mu,\nu}^{-1}\xi_j^{\nu},
\end{equation}
where $\mathbb{I}$ is the identity matrix and $t \in \mathbb{R}^+$ is a tuneable parameter. This model emerges as a continuous time limit ({\it i.e.} $\epsilon \sim dt$) of the unlearning rule \eqref{eq:unlearningpl}.
Most remarkably, it turns out that the maximal storage capacity increases as $t$ gets large.\footnote{Actually, the critical threshold found by \cite{DotsenkoDorotheyev} is approximately $1.07$. This is not to be meant as a violation of Garner's bound: the overflow is due to the  underlying replica-symmetry approximation.}
However, as $t$ gets larger and larger (which is the interesting limit in order to see the maximal capacity), the coupling matrix identically vanishes. As a result, on one side the retrieval region (see Fig.~\ref{fig:consVSdot}, right panel) is stretched toward higher values in $\alpha$ with respect to the Hopfield reference (see Fig.~\ref{fig:equivalenza}), but on the other side it is also confined to smaller values of $T$. This effect gets more pronounced as $t$ is increased, resulting in the total disappearance of the retrieval region. %
\par
One of the main contributions of the present work is to extend these unlearning approaches by simultaneously allowing also for reinforcement of the pure states \cite{RL1,RL2}. As we will see, this confers an extra-stability of these states against the fast noise, finally resulting in a sensibly enlarged and more robust retrieval region (with respect to the Hopfield reference and all the past extensions).
This result also suggests that, for a smart storage of information, remotion alone does not suffice: a suitable reinforcement is also in order.

\newpage

\section{{\em Unlearning$\&$Consolidating}: Focusing on Neurons}

\subsection{Model's definition, free energy and self-consistency equations}\label{replica-simmetric-theory}
Our investigation is based on the works by Personnaz, Guyon, Dreyfus \cite{Personnaz}, by Kanter and Sompolinksy \cite{KanterSompo}, and by Dotsenko et al. \cite{DotsenkoDorotheyev,DotsenkoTirozzi}. Along the same lines, we consider a network composed by $N$ neurons $\{ \sigma_i \}_{i=1,...,N}$, with $\sigma_i \in \{-1,+1\}$ $\forall i$, and $P$ patterns $\{\xi^{\mu}\}$, with $\xi_i^{\mu} \in \{-1,+1\}$ $\forall i,\mu$. Denoting with $t \in \mathbb{R}^+$ the sleep extent, we propose the following
\begin{Definition}
The {\em reinforcement$\&$removal} algorithm we propose has the following Hamiltonian representation:\footnote{As a matter of notation, we stress that the denominator $1/(\mathbb I+tC)$ in the generalized kernel is intended as the inverse matrix $(\mathbb I+tC)^{-1}$.}
\be\label{new-model}
H_{N,P}(\sigma|\xi,t)= - \frac{1}{2N}\sum_{i=1}^{N}\sum_{j=1}^{N}\sum_{\mu=1}^{P}\sum_{\nu=1}^{P}\xi_i^{\mu}\xi_j^{\nu}\left( \frac{1+t}{\mathbb{I}+t C} \right)_{\mu,\nu} \sigma_i \sigma_j,
\ee
where the $P$ patterns $\{\xi^{\mu}\}_{\mu=1,...,P}$, have $N$ binary entries $\xi_i^{\mu} \in \{-1,+1\}$, with $i \in (1,...,N)$, drawn from
$$
P(\xi_i^{\mu}=+1) = P(\xi_i^{\mu}=-1) = \frac12,
$$
and the correlation matrix is defined as
$$
C_{\mu,\nu} \equiv \frac{1}{N}\sum_{i=1}^{N}\xi_i^{\mu}\xi_i^{\nu}.
$$
\end{Definition}
Note that the interpretation of $t$ as the sleep extent is clear: for $t=0$  the system reduces to the standard Hopfield model, while for $t\rightarrow \infty$ the system approaches the pseudo-inverse matrix model (see the Appendix \ref{app:convergence} for the analytical proof). Remarkably, during the sleeping session, both reinforcement and remotion take place. In fact, in the generalized kernel appearing in \ref{new-model}, 
 the denominator ({\it i.e.}, the term $\propto (1+tC)^{-1}$) yields to the remotion of unwanted mixture states, while the numerator ({\it i.e.}, the term $\propto 1+t$) reinforces the memories. We refer to Secs. \ref{sec:separate} and \ref{sezione-quattro} for a more extensive discussion.
\par
In this Section,  we are instead interested in obtaining the phase diagram of our model, thus to compute explicitly -and extremize over the order parameters- the model's free energy (in the thermodynamic limit and under the replica symmetric assumption). The partition function of such a model is
\be
Z_{N,P}(\sigma|\xi,t)  = \sum_{\{\sigma\}} e^{-\beta H_{N,P}(\sigma|\xi,t)} = \sum_{\{ \sigma \}}\exp\left[ \frac{\beta}{2N }\sum_{i,j=1}^{N,N}\sum_{\mu,\nu=1}^{P,P}\xi^\mu _i \xi ^\nu _j \left( \frac{1+t}{\mathbb{I} +t  C}\right)_{\mu,\nu} \sigma_i \sigma_j\right].
\ee
by which we can introduce the main observable, namely
\begin{Definition}
The infinite volume limit of the intensive free energy $A(\alpha,\beta,t)$ associated to the model (\ref{new-model}) is defined as
\be
A(\alpha,\beta,t) =- \lim_{N \to \infty} \frac1{\beta N} \mathbb{E} \ln  Z_{N,P}(\sigma|\xi,t).
\ee
\end{Definition}
\begin{Remark}
The ``temporal variable'' $t$ within an (equilibrium) statistical mechanical theory may look weird, yet it should be noticed that the timescale for a sleeping session is much longer than the typical time scale for neuronal dynamics.\footnote{The latter, at least within a biological context, is fixed around $O(10^2)$ Hertz, namely the typical spiking time (considering also the absolute refractory period of a biological neuron).}
\end{Remark}

\subsubsection{Replica-symmetric scenario through statistical mechanics}

The replica symmetric assumption means that, in the thermodynamic limit $N \to \infty$, the order parameters self-average over their averages (denoted hereafter by a bar), {\it i.e.} $\lim_{N \to \infty}P(q)=\delta(q-\bar{q})$ and $\lim_{N \to \infty}P(\bold{m})=\delta(m- \bold{\bar{m}})$, in such a way that the related fluctuations can be discarded. Although this is a {\em reasonable} assumption, we actually know that in mean-field spin-glasses replica symmetry is broken at low temperatures. However, the effects of replica symmetry breaking are expected to be mild in associative neural networks \cite{Amit} and replica symmetry is the standard level of approximation in the statistical mechanical analysis of these models.
\par
Using the standard approach of replica technique ({\it i.e.} the so-called {\em replica trick} \cite{Amit,Coolen}), we write the large $N$ free-energy $A(\alpha,\beta,t)$
\be
A(\alpha,\beta,t) = -\lim_{ N \rightarrow \infty}\frac{1}{\beta N}\mathbb{E}\log Z_{N,P}(\sigma|\xi,t) =- \lim_{\substack{n\rightarrow 0 \\ N \rightarrow \infty}}\frac{\mathbb E Z_{N,P}(\sigma|\xi,t)^n -1}{\beta nN}.
\ee
Throughout the paper, we shall assume that the candidate pattern to be retrieved is $\xi^1$ and $\xi^\mu$ for $\mu \ge 2$ contribute to the slow noise, therefore here $\mathbb E$ is the average over the $P-1$ not-retrieved patterns. The quenched average of the replicated partition function can be represented in Gaussian integral form as
\be\label{eq:boltzmann}
\begin{split}
\mathbb E Z_{N,P}(\sigma|\xi,t)^n &=\mathbb E \mathcal \prod_{n=1}^\alpha \mathcal{C} \sum_{\{ \sigma^1 \} }\dots \sum_{\{ \sigma ^n \} } \int \Big(\prod_{\mu, \alpha=1}^{P,n}Dz_\mu^\alpha\Big)\Big(\prod_{i,\alpha=1}^{N,n}D\phi_i^\alpha\Big)\cdot \\ &\cdot \exp\Big(\sqrt{\frac{\beta}{N} (t+1)}\sum_{\mu, i, \alpha=1}^{P,N,n}z_\mu ^\alpha \xi^\mu _i \sigma_i ^\alpha +i\sqrt{\frac{t}{N}} \sum_{\mu, i, \alpha=1}^{P,N,n}z_\mu ^\alpha \xi^\mu _i \phi _i ^\alpha \Big),
\end{split}
\ee
where $\mathcal P (z_\mu ^\alpha)= \mathcal P (\phi_i ^\alpha)= \mathcal N(0,1)$ and $\mathcal C=\det^{1/2}(\mathbb I + t C)$ is a normalization constant compensating the prefactors of the Gaussian integrations and trivially contributing to the free energy. %
This model strongly resembles \cite{DotsenkoDorotheyev}, with the only difference that - in the first term - here we have $\beta(1+t)$ (instead of $\beta$) realizing an optimal tuning between the two 2-body couplings of the relevant variables. As we will see, this scaling is crucial to keep the thermodynamic properties of the model stable, since it ensures that the critical temperature at zero load stays fixed at $\beta_c=1$ as $t$ is tuned. Thus, this interpolation between the Hopfield model and the pseudo-inverse one automatically prevents the collapse of the retrieval region on the horizontal axis in the phase diagram. %
In the Appendix \ref{app:replica}, we report in details the calculations of the free energy $A(\alpha,\beta,t)$  of the model, while here we provide just the explicit expression (ignoring trivial contributions):
\be\label{eq:fgeneral}
\begin{split}
	A(\alpha,\beta,t) &= \frac{1}{2n(1+t)}\sum_{\alpha=1}^n (m_1 ^\alpha)^2+\frac{\alpha \beta}{2n}\sum_{\alpha,\beta=1}^{n,n} p_{\alpha\beta}q_{\alpha\beta}+\frac{\alpha}{2n\beta}\log\det \left[ \mathbb I-\beta(1+t)\hat q \right]\\
&-\frac{1}{n\beta}\mathbb E\ln \sum_{ \{ \sigma \} }  \int\Big(\prod_{\alpha=1}^n D\phi^{\alpha}\Big)\exp\Bigg[\beta \sum_{\alpha=1}^n m_1^\alpha \xi^1\Big(\sigma^\alpha+i\sqrt{\frac{t}{\beta(1+t)}}\phi^\alpha\Big)\\
&+\frac{\alpha \beta^2}{2}\sum_{\alpha,\beta=1}^{n,n}p_{\alpha\beta}\Big(\sigma^\alpha+i\sqrt{\frac{t}{\beta(1+t)}}\phi^\alpha\Big)\Big(\sigma^\beta+i\sqrt{\frac{t}{\beta(1+t)}}\phi^\beta\Big)\Bigg],
\end{split}
\ee
where $m_1^\alpha$ is the Mattis magnetization (of the $\alpha$-th replica) associated to the pattern $\xi^1$ to be retrieved, $\hat q$ is the overlap matrix whose element $q_{\alpha\beta}$ is the generalized overlap between the replicas of the system (labeled with $\alpha$ and $\beta$), and it is defined as \cite{DotsenkoDorotheyev}
\be\label{eq:overlap}
q_{\alpha\beta}= \frac{1}{N}\sum_i \Big(\sigma^\alpha+i\sqrt{\frac{t}{\beta(1+t)}}\phi^\alpha\Big)\Big(\sigma^\beta+i\sqrt{\frac{t}{\beta(1+t)}}\phi^\beta\Big),
\ee
with $p_{\alpha\beta}$ its conjugate variables \cite{Coolen}.
\par
Imposing replica symmetry
\begin{subequations}
	\begin{align}
	m_1^\alpha &=m\quad \forall \alpha, \\
	q_{\alpha\beta} &=Q\delta_{\alpha\beta}+q(1-\delta_{\alpha\beta}), \\
	p_{\alpha\beta} &=P\delta_{\alpha\beta}+p(1-\delta_{\alpha\beta}),
	\end{align}
\end{subequations}
after straightforward computations, we can finally state the next
\begin{Proposition}
The infinite volume limit of the replica-symmetric free energy for the model (\ref{new-model}), expressed in terms of the order parameters $m$ and $q$, reads as
\be\label{eq:frsa}
\begin{split}
A(\alpha,\beta,t) &=\frac{m ^2}{2(1+t)} \Big(1+\frac{t}{\Delta}\Big)+\frac{(1+t)(\Delta-1)}{2t}Q + \frac{\alpha\beta }{2}p (Q-q)\\&
+\frac{\alpha}{2\beta}\Big(\log[1-\beta(1+t)(Q-q)]-\frac{q\beta(1+t)}{1-\beta(1+t)(Q-q)}\Big)+\frac{(1+t)(1-\Delta)}{2 t \Delta}\\&+
\frac{\log \Delta}{2\beta}+\frac{\alpha p t}{2(1+t)\Delta}-\frac{1}{\beta} \int Dx \log \cosh \Big[\frac{\beta }{\Delta}(m+\sqrt{\alpha p}x)\Big]-\frac{\log 2}{\beta},
\end{split}
\ee
where $Dx$ is the Gaussian measure and $\Delta= 1+\alpha \beta t(1+t)^{-1}(P-p)$.
\end{Proposition}
The self-consistency equations for the model \ref{new-model} are derived by imposing the extremal condition for the free energy \ref{eq:frsa} with respect to the five order parameters, so we arrive at the following
\begin{Proposition}
The self-consistency equations read
\begingroup\makeatletter\def\f@size{9.5}\check@mathfonts
\begin{subequations}\label{eq:sceqs}
	\begin{align}
	m &=\frac{1+t}{\Delta+t}\int Dx \tanh\Big[\frac{\beta}{\Delta}(m+\sqrt{\alpha p}x)\Big],\label{eq:a} \\
	p &=\frac{q(1+t)^2}{[1-\beta(1+t)(Q-q)]^2}, \\
	\Delta &=1+\frac{\alpha t}{1-\beta(1+t)(Q-q)},\label{eq:cc}  \\
	\label{eq:d}
	q &=Q+\frac{t}{\beta (1+t)\Delta}-\frac{1}{\Delta^2}\int Dx \cosh^{-2}\Big[\frac{\beta}{\Delta}(m+\sqrt{\alpha p}x)\Big], \\
	Q \Delta^2 &=1-\frac{t\Delta}{\beta(1+t)}+\frac{\alpha p t^2}{(1+t)^2} -\frac{m^2 t(t+2\Delta)}{(1+t)^2}
	-\frac{2\alpha\beta p t}{(1+t)\Delta}\int Dx \cosh^{-2}\Big[\frac{\beta}{\Delta}(m+\sqrt{\alpha p}x)\Big].
	\end{align}
\end{subequations}
\endgroup
\end{Proposition}
By studying these equations it is possible to derive the phase diagram related to the cost-function (\ref{new-model}). This point will be achieved in Sec. \ref{sec:RS-FD}.
\begin{Remark}
For $t \to 0$, both the free energy \eqref{eq:frsa} and the self-consistency equations \eqref{eq:sceqs} reduces to the AGS ones as they should.
\end{Remark}

\subsection{Remotion {\it or} Reinforcement: a separate analysis}\label{sec:separate}
In order to better analyze the structure of our model, we split the whole Hamiltonian \eqref{new-model} in two by considering separately the contributions coming from the numerator (reinforcement) and the denominator (remotion) in the generalized kernel. 
In other words, we take into account the following cost-functions separately to show that, when isolated, none of them constitutes a major breakthrough, that appears solely when these two features are left to work together (as we will prove later).
\begin{subequations}
\begin{align}
H^{(1)}_{N,P}&= -\frac{1}{2}\sum_{\mu} \sum_{ij} \xi^\mu _i \xi^\mu _j (1+t) \sigma_i \sigma_j,\\
H^{(2)}_{N,P}&= -\frac{1}{2}\sum_{\mu\nu}\sum_{ij} \xi^\mu_i \xi^\nu_j (\mathbb I + t C)^{-1} _{\mu,\nu} \sigma_i \sigma_j.
\end{align}
\end{subequations}
\begin{itemize}
\item Concerning the first cost-function, due to reinforcement, it is evident that the only net effect it may induce (when playing along) is to stretch the minima landscape by amplifying the energetic gaps by a factor $(1+t)$.  The model is formally identical to the Hopfield one with a rescaled thermal noise $\tilde \beta=\beta(1+t)$: this implies that the zero-capacity critical temperature is given by $\tilde T_c = \tilde \beta_c ^{-1} = 1$, namely $T_c=(1+t)$. See Figure \ref{fig:consVSdot} (left panel).

\item Concerning the second cost-function, due to mixtures removal, this is precisely  the coupling matrix \eqref{eq:J_D}  whose statistical mechanics has been deeply analyzed in \cite{DotsenkoDorotheyev} (in the standard replica-symmetric regime). This model emerges as a continuous-time limit of the unlearning procedure analyzed by Plakov and Semenov \cite{Semenov1} and it is thus natural to link this model to unlearning features. An important point is that the zero-capacity critical temperature for the transition between the retrieval and spin-glass phases is $T_c=(1+t)^{-1}$, therefore in the large unlearning time limit the former is mashed on the $\alpha$ axes and, actually, no robustness is retained. See Figure \ref{fig:consVSdot} (right panel).

\end{itemize}

\begin{figure}[t!]
	\centering
	\begin{minipage}[c]{.49\textwidth}
		\centering
		\includegraphics[width=\textwidth]{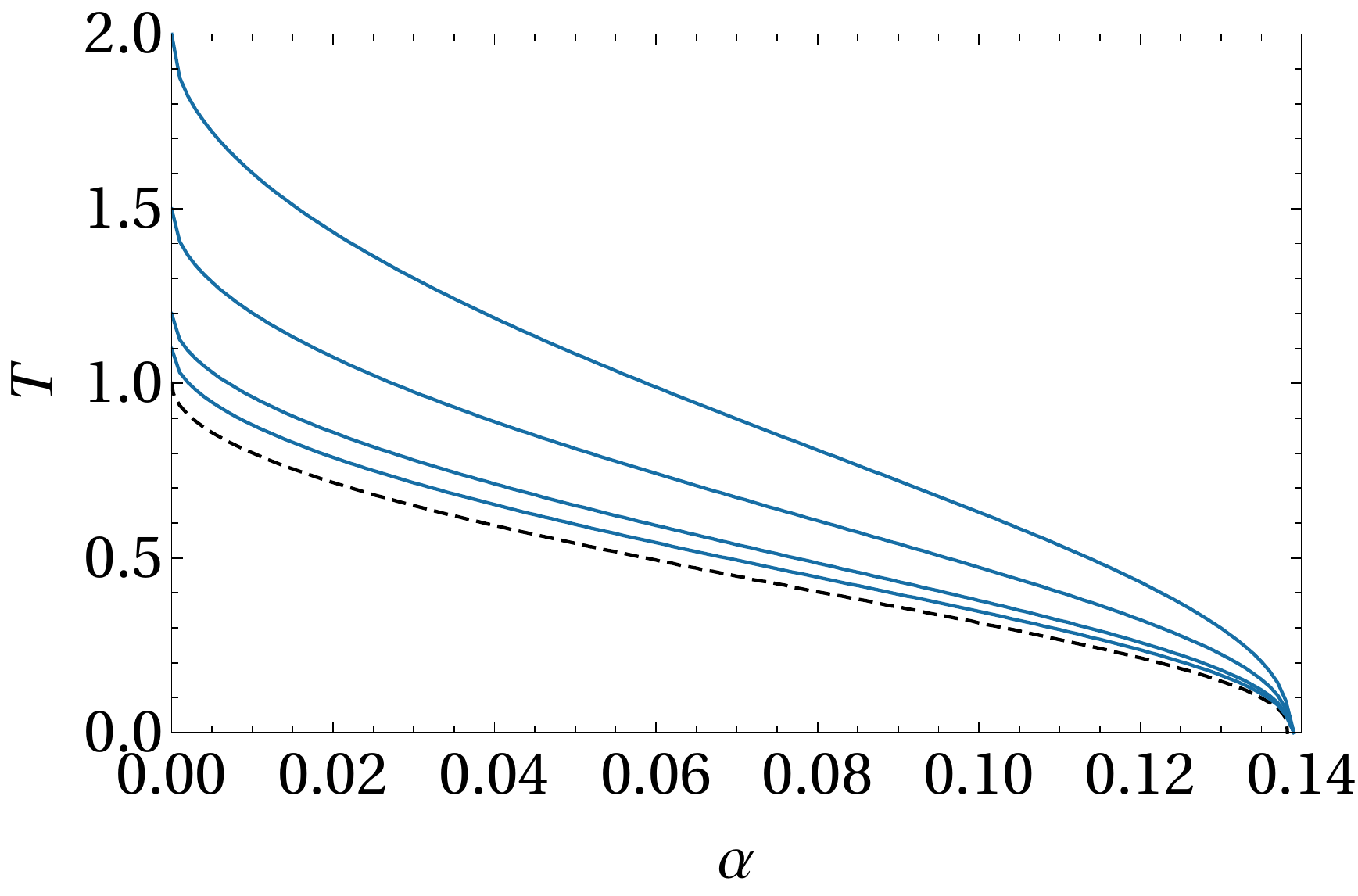}
	\end{minipage}
	\begin{minipage}[c]{.49\textwidth}
		\centering
		\includegraphics[width=\textwidth]{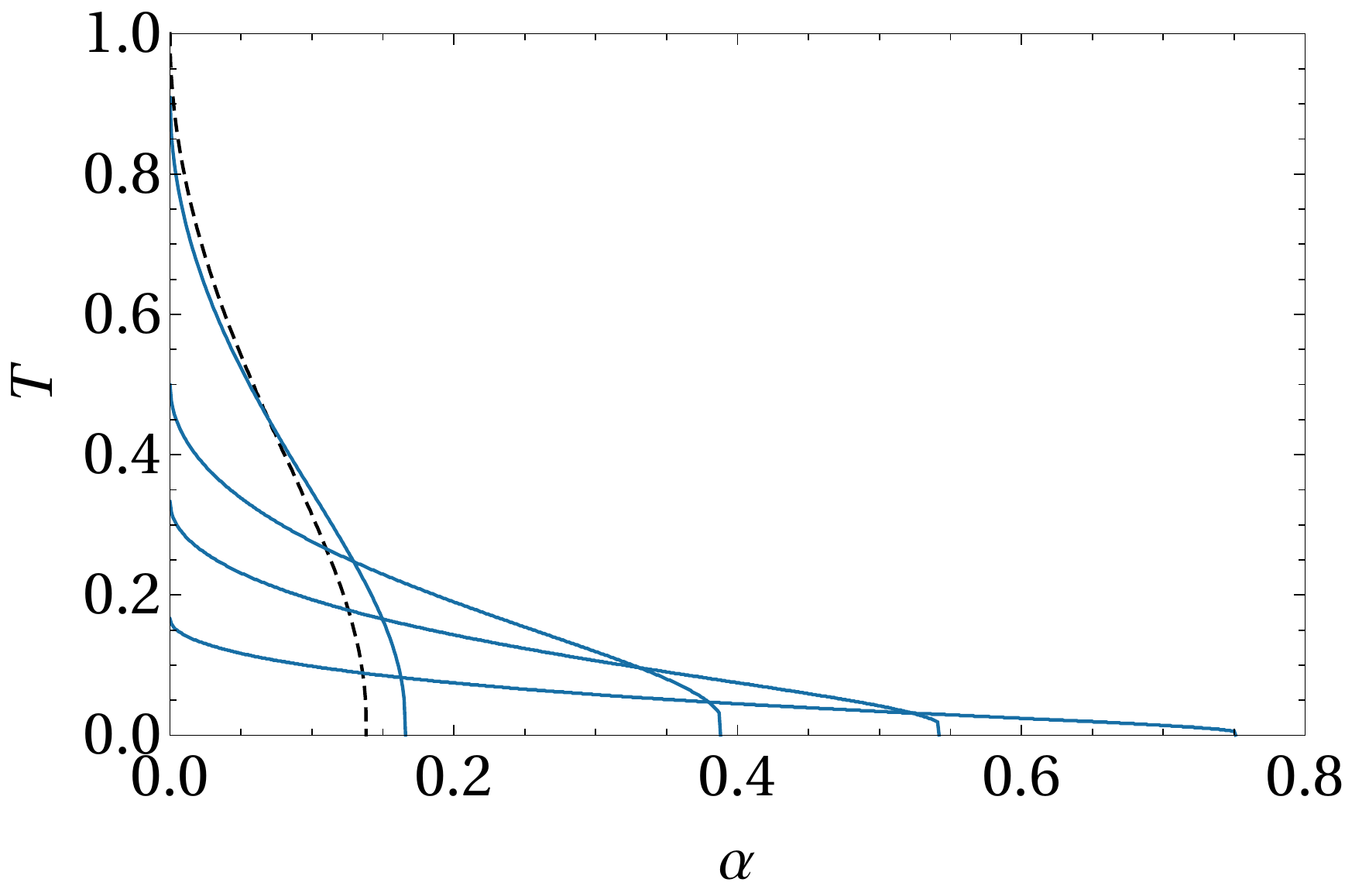}
	\end{minipage}
	\caption{{\bfseries Reinforcing and unlearning models.} Left: the plot shows the retrieval regions for the reinforcing model $H^{(1)}$ for $t=0$ (Hopfield), $0.1$, $0.2$, $0.5$ and $1$. The critical temperature in the zero-capacity limit is $T_c=(1+t)$ and this trivial shift in the critical temperature is the solely novelty of this model. Right: the plot shows the retrieval regions for the Dotsenko model as also discussed in \cite{DotsenkoDorotheyev}. The critical temperature grows with $t$, by the critical temperature in the zero-capacity limit decreases as $T_c =(1+t)^{-1}$, so that the retrieval regions are mashed on the horizontal axes.}\label{fig:consVSdot}
\end{figure}
With these ideas in mind, it is also reasonable to expect that - in the full model \eqref{new-model} - the mashing effect of unlearning can be compensated by the rescaling of the thermal noise, therefore giving an optimal balance between the Reinforcement and the Removal features. The evaluation of the phase diagram for our model is presented in the next Section.

\subsection{Zero-temperature (noise-less) critical capacity}
The first point we would like to analyze is the critical capacity in the vanishing temperature limit ($\beta \to \infty$). As standard in this case, it is convenient to introduce $c\equiv\beta(Q-q)$ quantifying the difference between diagonal and non-diagonal replica overlaps. From Eq.~\ref{eq:cc} it s easy to verify that this quantity satisfies the self-consistency equation
\be \label{eq:c}
c=\frac{\beta}{\Delta^2}\int Dx \cosh^{-2}\Big[\frac{\beta}{\Delta}(m+\sqrt{\alpha p}x)\Big]-\frac{t}{ (1+t)\Delta}.
\ee
Using the equation for $\Delta$ in the zero temperature limit, with simple arguments it is easy to check that $c$ is finite and, consequently, $q\rightarrow Q$ as $T \to 0$.
Now, since the hyperbolic tangent in \eqref{eq:a} tends to the error function,
after some rearrangement we end with the simplified set of equations
\begin{subequations}
	\begin{align}
	m&=\frac{1+t}{\Delta+t}\text{erf}\left(\frac{m}{\sqrt{2 \alpha p}}\right),\\
	p&=\frac{Q(1+t)^2}{[1-(1+t)c]^2},\\
	\Delta &= 1+\frac{\alpha t}{1-(1+t)c},\\
	c&=\frac{1}{\Delta}\sqrt{\frac{2}{\pi \alpha p}}\exp\left(-\frac{m^2}{2\alpha p}\right)-\frac{t}{\Delta(1+t)},\\
	Q\Delta^2 &= 1+\frac{\alpha p t^2}{(1+t)^2}-\frac{m^2 t (t+2\Delta)}{(1+t)^2}-\frac{2\alpha t}{1+t}\sqrt{\frac{2}{\pi \alpha p}}\exp\left(-\frac{m^2}{2\alpha p}\right).
	\end{align}
\end{subequations}
\begin{figure}[b!]
	\centering
	\begin{minipage}[c]{.49\textwidth}
		\centering
		\includegraphics[width=\textwidth]{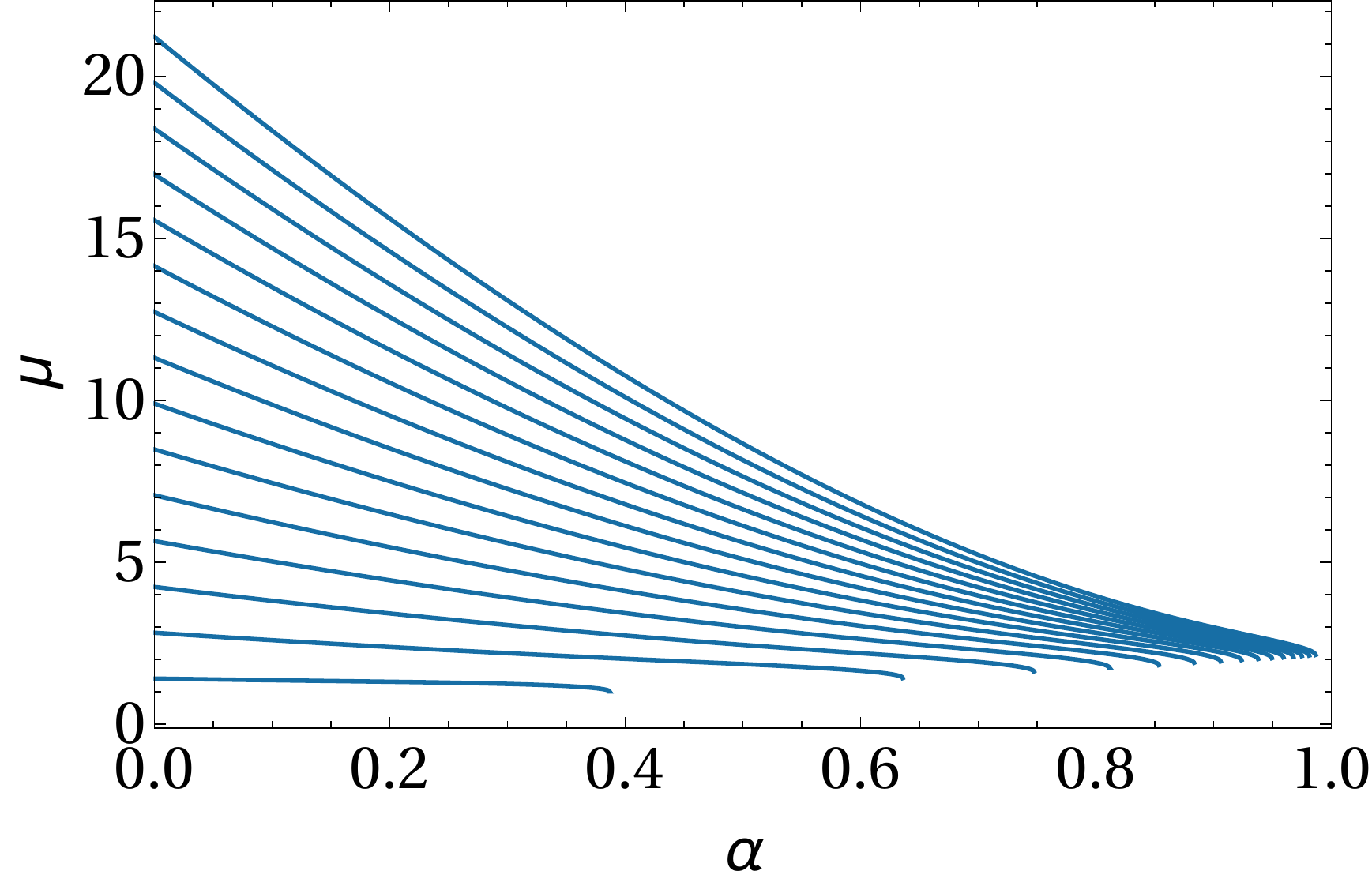}
	\end{minipage}
	\begin{minipage}[c]{.49\textwidth}
		\centering
		\includegraphics[width=\textwidth]{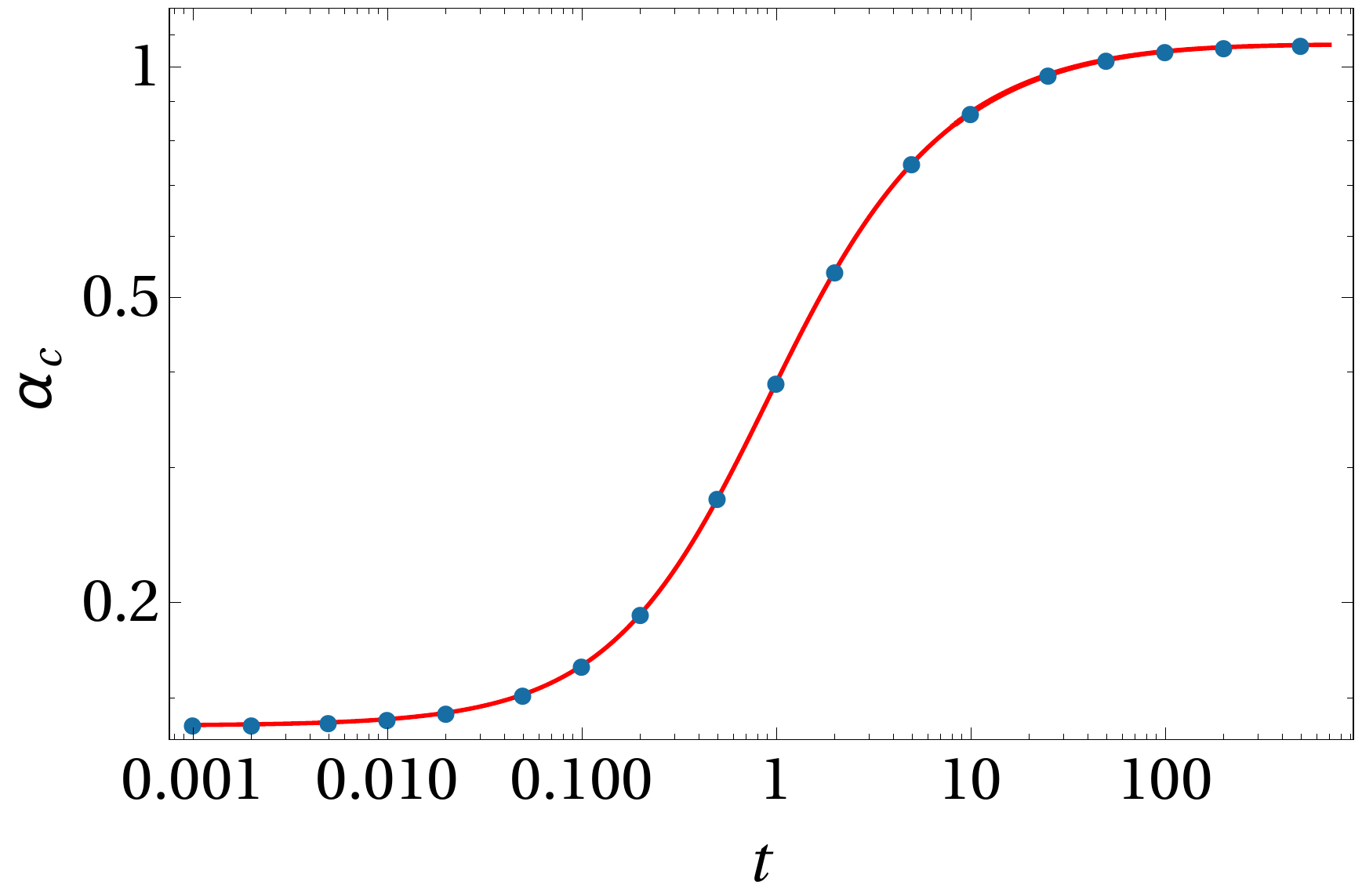}
	\end{minipage}
	\caption{{\bfseries Zero-temperature analysis of the critical capacity.} Left panel: numerical solutions for $\mu$ of the self-consistency equations in the zero temperature limit (\ref{eq:sc_T0}) for several unlearning times: $t=1,3,\dots, 29$. Right panel: temporal dependence of the critical capacity at zero temperature. The blue dots represent the storage capacity beyond which the only possible solution has $\mu=0$, {\it i.e.} the end-points of the curves in the left plot). The red curve is the fit given by $y = x/(x + a)$, with $a=2.84 \pm 0.01$ obtained by first normalizing data in $[0,1]$, namely $\alpha_c \rightarrow [\alpha_c-\min(\alpha_c)]/[\max(\alpha_c) - \min(\alpha_c)]$.}\label{fig:zeroT}
\end{figure}
Introducing the quantities $\mu = {m}({2p})^{-1/2}$, $\Pi= p^{-1/2}$ and eliminating $Q$, we recast the original set of equations as
\begin{subequations} \label{eq:sc_T0}
	\begin{align}
	\mu &= \frac{\Pi}{\sqrt{2}}\frac{1+t}{\Delta+t}\text{erf}\left(\frac{\mu}{\sqrt{\alpha}}\right),\\
	\Delta&= 1+\frac{\alpha t}{1-(1+t)c},\\
	c&= \frac{\Pi}{\Delta}\sqrt{\frac{2}{\pi \alpha}}\exp\left(-\frac{\mu^2}{\alpha}\right)-\frac{t}{\Delta(t+1)},\\
	\Delta^2[1-(1+t)c]^2&=\Pi^2 (1+t)^2+\alpha t^2 -2\mu^2 t (t+2\Delta)-2\alpha t (1+t)\Pi \sqrt{\frac{2}{\pi \alpha}}\exp\left(-\frac{\mu^2}{\alpha}\right).
	\end{align}
\end{subequations}
Since $\mu$ is proportional to $m$ and since $p$ is finite for any $t$, solutions with $\mu \neq 0$ correspond to retrieval solutions (since $m \neq 0$). Searching for solutions of \eqref{eq:sc_T0} with $\mu\neq0$ for given $t$ is therefore equivalent to determine the upper bound for the storage capacity $\alpha_c (t)$. We solved these equations numerically\footnote{Note that, for $\alpha \sim 0$, we have the behaviors $\Delta\sim1$, $c\sim -t/(t+1)$, $\mu \sim 2^{-1/2}(1+t)$ and $\Pi \sim 1+t$.} for $t \in (1,1000)$ and we reported the solutions in Fig. \ref{fig:zeroT} (left panel). The end point of each curve separates the $\alpha$ axis in the regions with respectively $\mu\neq 0$ and $\mu =0$, therefore identifying the critical capacity for each fixed $t$ value. We report the critical capacity $\alpha _c$ as a function of the sleep extent $t$ in the right plot (blue dots) of Fig. \ref{fig:zeroT}.
The $t\rightarrow 0$ limit provides the critical capacity $\alpha_c (t=0) \sim 0.138$, that correctly recovers the standard Hopfield result \cite{AGS1}, while in the opposite limit $t \rightarrow \infty$ we have the upper bound $\alpha_c \sim 1.07$, in perfect agreement with \cite{DotsenkoDorotheyev}. It is worth stressing that the gap between the critical capacity obtained here ($\alpha_c  \sim 1.07$) and the maximal critical capacity according to Gardner's theory ($\alpha_c =1.00$) should be ascribed to the replica symmetry breaking expected to take place in this model (this was already pointed out in \cite{DotsenkoDorotheyev}).
Interestingly, the critical capacity displays a log-sigmoidal growth in $t$. This suggest that the intrinsic scale for $t$ is logarithmic: relatively small values of $t$ already provide a critical threshold $\alpha_c$ close to $1$; more precisely, $\alpha_c(t=1) \approx 0.4$ and $\alpha_c(t=5) \approx 0.8$. The log-sigmoidal shape also gives hints for convenient choice of the sleep extent: assuming that we want the best possibile capable machine, increasing $t$ is somehow expensive ({\it e.g.} in terms of time), then the region corresponding to the flex of the curve (approximately $t = 1$) is where a small increase in $t$ determines the largest return in terms of capacity.

\subsection{Replica symmetric phase diagram}\label{sec:RS-FD}

We solved the set \eqref{eq:sceqs} of five self-consistency equations numerically for different values of $t$. We used these solutions to build the phase diagram depicted in Figs.~\ref{fig:criticallines} and \ref{fig:phased}. A comparison with the phase diagram from AGS theory and Dotsenko {\it et al.} (Figs.~\ref{fig:equivalenza} and \ref{fig:consVSdot}, respectively) shows that our model  displays the same qualitative behaviors ({\it i.e. }spin-glass, mixed retrieval and pure retrieval phases), but their boundary lines significantly depend on $t$. In particular, the retrieval region gets wider with $t$.\footnote{On the contrary, in the model studied by Dotsenko {\it et al.} \cite{DotsenkoDorotheyev}, the area of the retrieval region decreases as $t$ grows, and vanishes in the large unlearning time limit. This is clear by noticing that the critical capacity at zero thermal noise level and $t \rightarrow \infty$ reaches the fixed value $\alpha_c \sim 1.07$, while the critical temperature at zero capacity is $T_c =(1+t)^{-1}$, so it vanishes in the $t\rightarrow \infty$. Since the critical curve characterizing the phase transition to the spin-glass phase is (from thermodynamics argument) a monotonous decreasing as a function $T(\alpha)$, it immediately follows that the retrieval region area becomes smaller and smaller for increasing $t$.} 
More precisely, we distinguish the following transition lines:

\begin{figure}[!]
	\centering
	\begin{minipage}[c]{.7\textwidth}
		\centering
		\includegraphics[width=\textwidth]{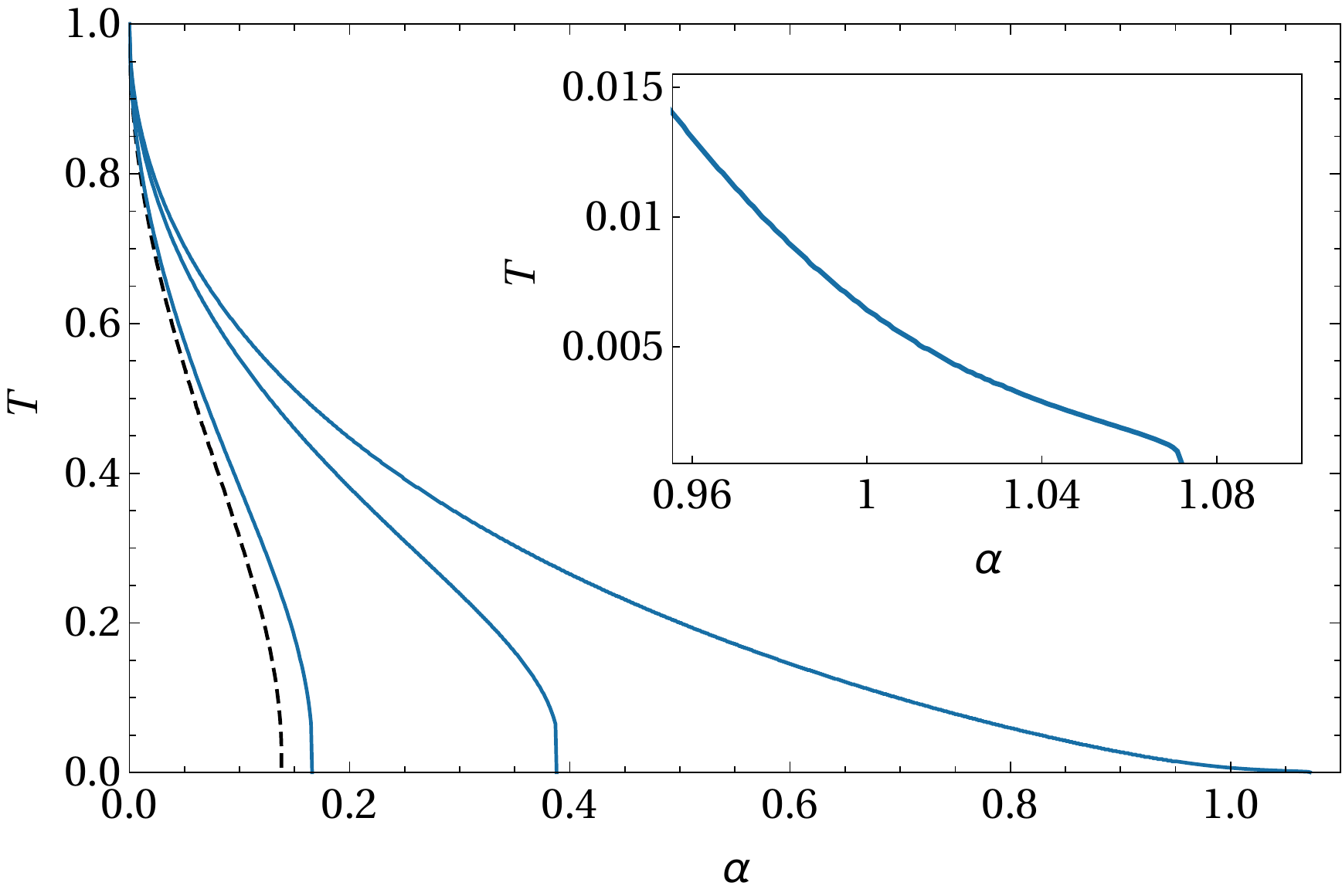}
	\end{minipage}%
	\caption{Critical line for the transition between retrieval and spin-glass phases for various values of the unlearning time. From the left to the right: $t=0$ (Hopfield, black dashed line), $0.1$, $1$ and $1000$. The inner plot on the top-right corner shows the tail of the critical curve for $t=1000$.}\label{fig:criticallines}
\end{figure}
\begin{figure}[t!]
	\centering
	\begin{minipage}[c]{.7\textwidth}
		\centering
		\includegraphics[width=\textwidth]{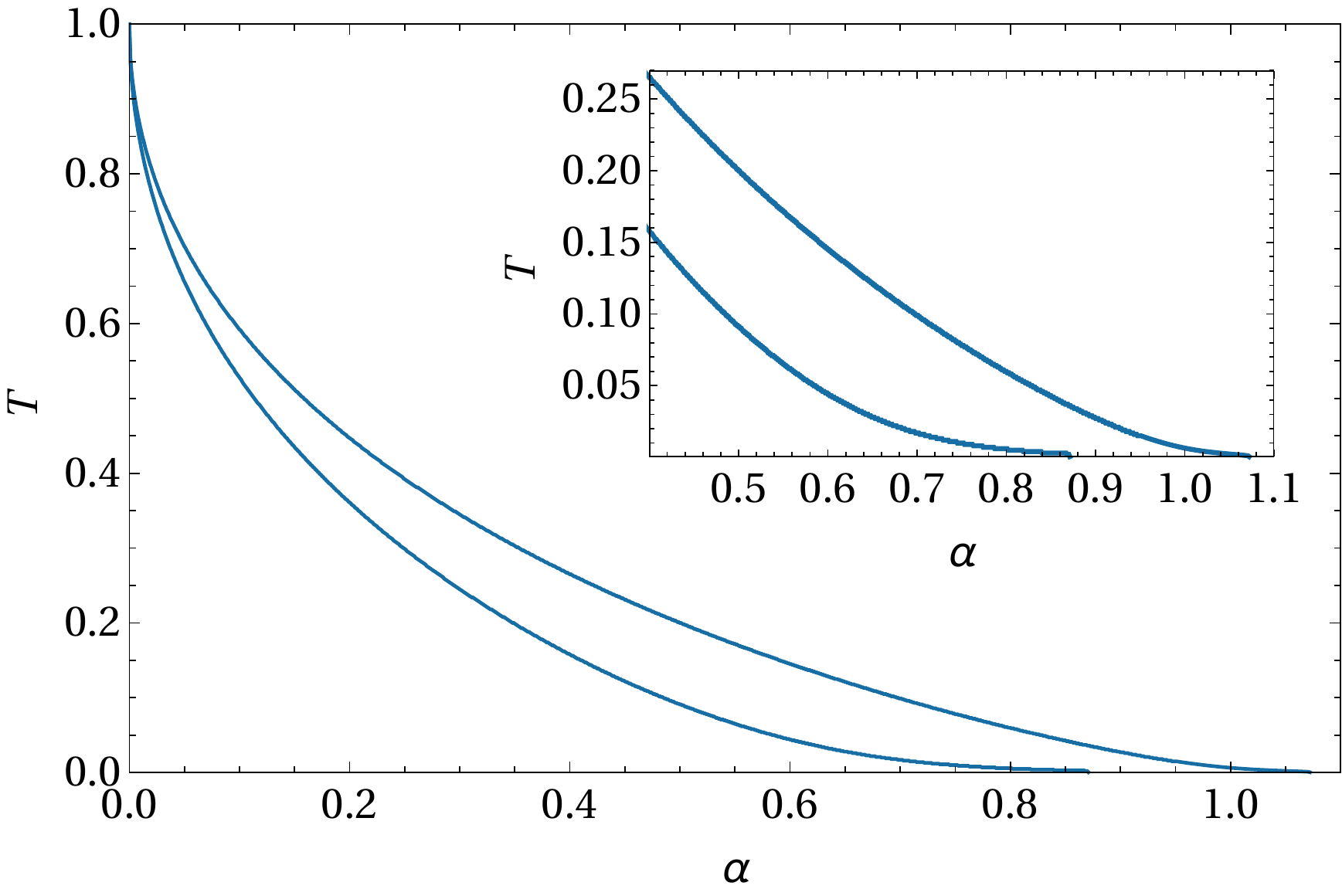}
	\end{minipage}%
	\caption{Phase diagram in the large unlearning time limit ($t=1000$). The two curves trace the boundary of the maximal retrieval regions where patterns are global free energy minima (inner boundary) or local free energy minima (outer boundary). The inner plot on the top-right corner shows the tails of both the critical curves. We stress that, as already pointed out by Dotsenko$\&$Tirozzi, the extension of the retrieval region in the low-temperature regime up to $\alpha_c \sim 1.07$ is just a chimera of the replica symmetric approximation, while in the true RSB phase $\alpha_c \to 1.00$, according to Gardner's theory \cite{Gardner}.}\label{fig:phased}
\end{figure}
\begin{figure}[t!]
	\centering
	\begin{minipage}[c]{.49\textwidth}
		\centering
		\includegraphics[width=\textwidth]{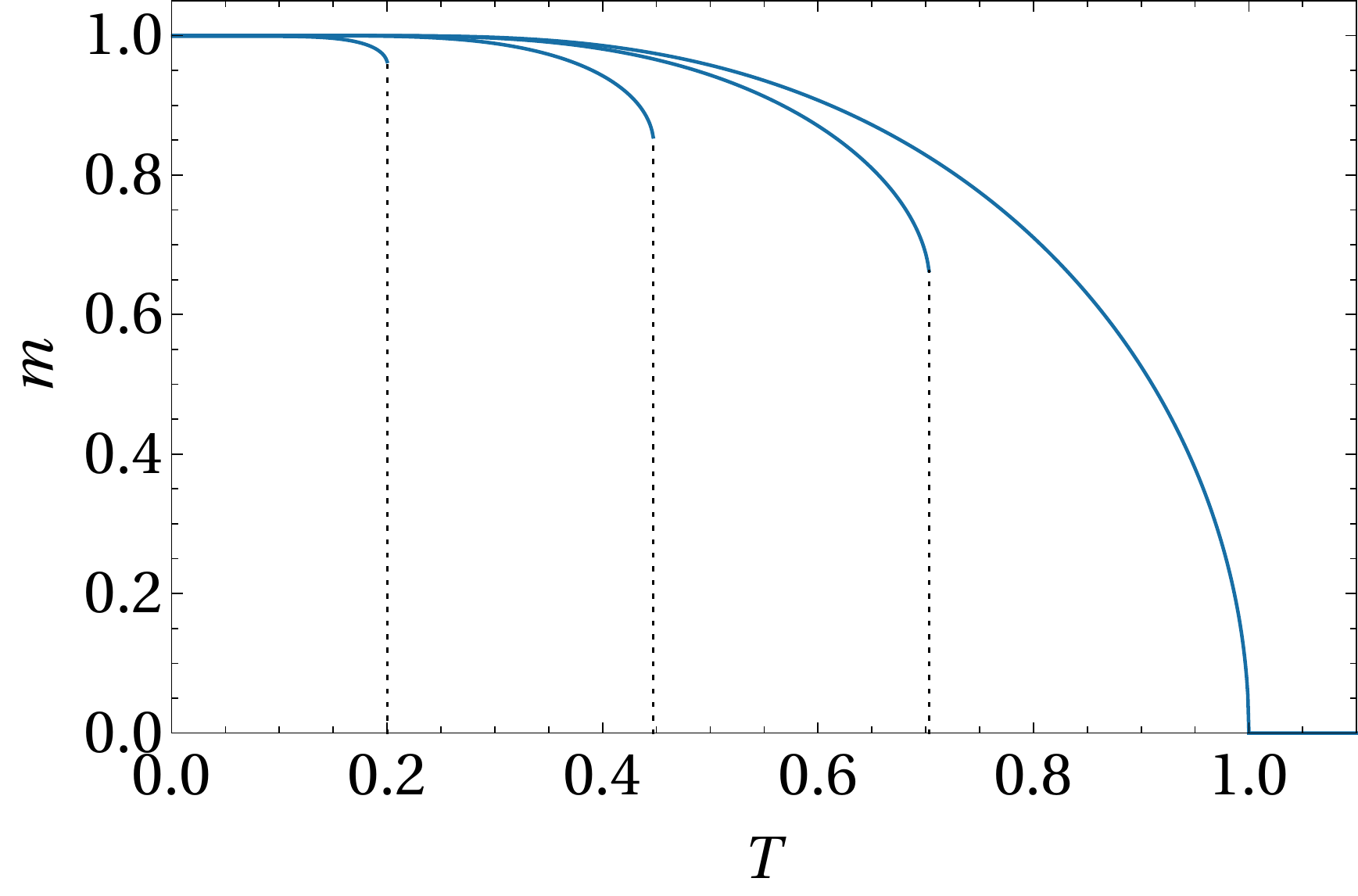}
	\end{minipage}%
	\hspace{2mm}%
	\begin{minipage}[c]{.49\textwidth}
		\centering
		\includegraphics[width=\textwidth]{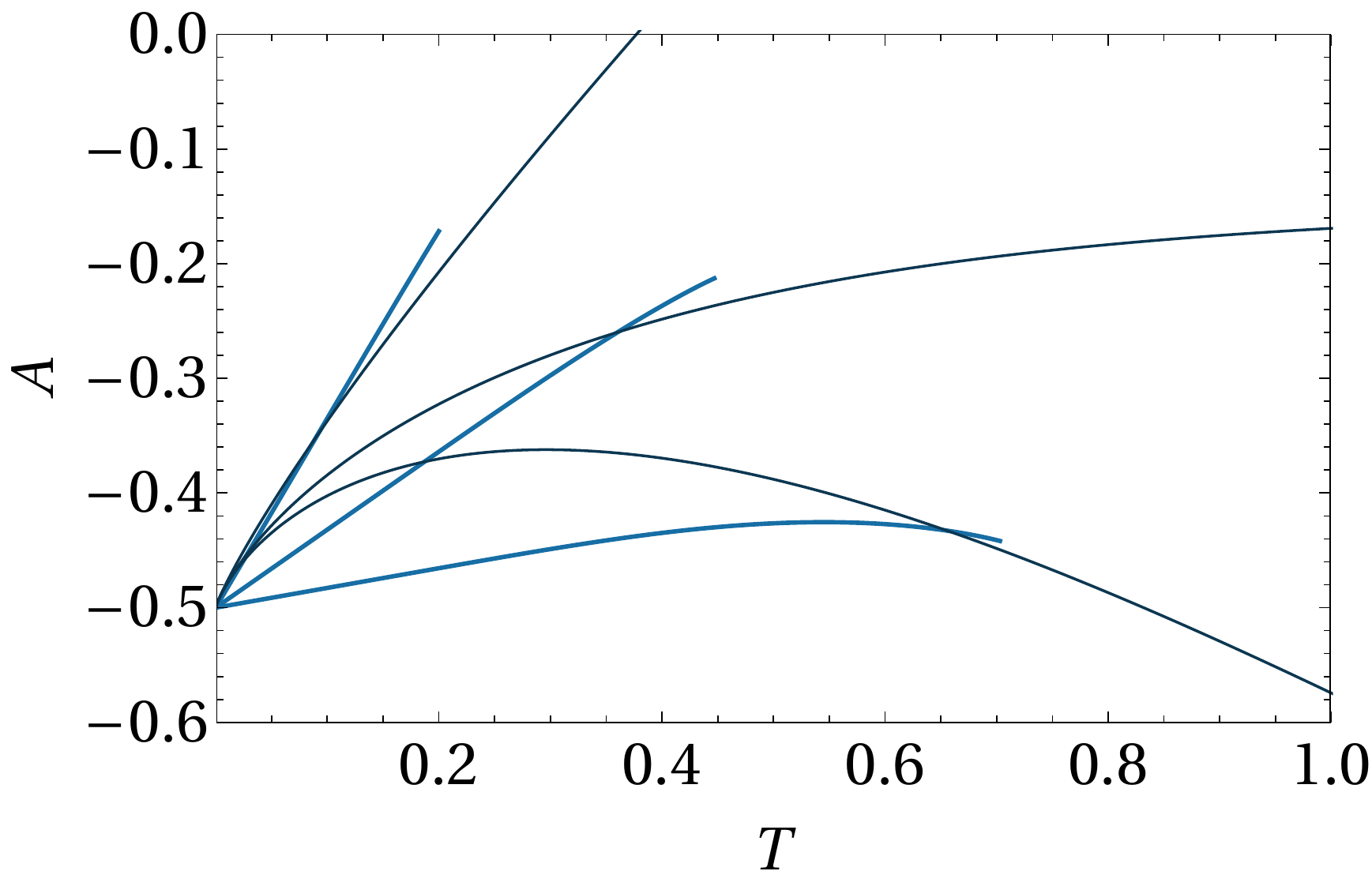}
	\end{minipage}
	\caption{{\bfseries Mattis magnetization and free-energy for $t=1000$.} Left: the plot shows the Mattis magnetization $m$ as a function of the temperature for various storage capacity values ($\alpha=0$, $0.05$, $0.2$ and $0.5$, going from the right to the left). The vertical dotted lines indicates the jump discontinuity identifying the critical temperature $T_c(\alpha)$ which separates the retrieval region from the spin-glass phase. Right: the plot shows the free-energy as a function of the temperature for various storage capacity values ($\alpha=0.05$, $0.2$ and $0.5$, going from the bottom to the top) in the retrieval (thicker light blue lines) and spin-glass (thinner dark blue lines) states.}\label{fig:jumpVST}
\end{figure}

\begin{itemize}
\item \emph{Spin-glass versus Mixed retrieval region}. Here, we focus on the transition between the retrieval region and the spin glass phase, therefore searching for the critical curve $T_c (\alpha)$ beyond which the only possible solution has $\bold{m}=0$ with $q \neq 0$. The situation we found is formally similar to the original Hopfield model: in the low-storage limit ({\it i.e.} $\alpha=0$), the replica-symmetric free-energy is continuous everywhere and differentiable {\it almost} everywhere (with the only exception being the critical point $T_c=1$ as expected, where we have a second-order phase transition in the standard Ehrenfest classification). For higher values of the capacity $\alpha>0$, the phase transition turns out instead  to be of the first kind, with a discontinuity taking place at the critical temperature $T_c(\alpha)$. Left plot in Fig. \ref{fig:jumpVST} shows an example of this behavior for various values of the storage capacity $\alpha$ and $t=1000$. The jumps of the magnetization versus $T$ take place on the critical line separating the retrieval region from the spin glass phase, so that we can study the occurrences of these jumps to reconstruct the phase boundary of the retrieval region. The results have been collected in Figure \ref{fig:criticallines} for various sleep extents. By inspecting the plot, it clearly emerges that the critical storage capacity $\alpha$ effectively increases with the sleeping session, with the zero-capacity critical temperature $T_c(\alpha=0)$ being stable to $1$. Thus, our interpolation scheme effectively leads to an increase of the retrieval performances offering a working tradeoff between unlearning spurious memories and consolidating pure ones.

\item \emph{Mixed retrieval versus Pure retrieval region}. The region where the pure states are global minima for the free-energy is identified by solving the self-consistency equations (with fixed $\alpha$ and $T$) for both retrieval ($m\neq 0$) and spin-glass ($m=0$) states and then comparing the values of the corresponding free-energies. Right plot in Fig. \ref{fig:jumpVST} shows the behavior of free-energy for both these solutions for various storage capacity. The intersection point between the corresponding curves identifies the critical temperature $T_R(\alpha)$ below which the pure states (globally) minimize the free-energy. We performed this analysis for $t=1000$. The result is the phase diagram depicted in Fig. \ref{fig:phased}.
\end{itemize}

A visual comparison between the Hopfield phase diagram (Fig. \ref{fig:equivalenza}) and the present one (Fig. \ref{fig:phased}) immediately evidences the crucial role of sleeping for improving the network performances.

\subsection{Numerical results}
In this Section, we inspect numerically some aspects of the exposed theory which are too difficult to control analytically. In particular, we want to check that our replica-symmetric ansatz is reasonable, by comparing its predictions with Monte Carlo simulations (where no assumptions are made). Then, we want to analyze the field distributions $h_i$ and the robustness of the attraction basins of the pure minima.

\subsubsection{Checking the Replica Symmetric assumption}
We performed Monte Carlo simulation to mimic the evolution of a finite-size network made of $N$ neurons and $P$ patterns. For a given realization of patterns $\xi_i^{\mu}$, $i=1,...,N$ and $\mu=1,...,P$, for a given temperature $T=1/\beta$ and for a given sleeping time $t$, we let the system evolve by a single spin-flip Glauber dynamics and, once the equilibrium state is reached,\footnote{This can be checked by evaluating the stability of observables and the width of their fluctuations} we measure the thermal average of the Mattis magnetization, referred to as $\bar{\mathbf{m}}$. This is repeated for $M$ realizations of the patterns over which thermal averages are accordingly averaged. The resulting value provides our numerical estimate for the Mattis overlaps. Different parameters $(N,P,\beta,t)$ are considered and, for each choice, the same procedure applies. A sample of results is shown in Fig.~\ref{fig:Ele}, where one can check that, as $t$ increases, the Mattis magnetization $m_1$ corresponding to the retrieved pattern $\xi^1$ vanishes at larger values of $T$ and $\alpha$ (with a slight abuse of notation here we mean $\alpha=P/N$). Remarkably, these results are also quantitatively consistent with those presented in Fig. \ref{fig:criticallines}. Since these results were obtained from simulations at finite size and without asking for replica symmetry, this check strongly corroborates the analytical findings.
Further, in Fig.~\ref{fig:FSS} we compare outlines pertaining to systems of different sizes $N$, but same choice of $\beta, t, \alpha$; the theoretical expression found in Eq.~\ref{eq:sceqs} is also depicted. Finite-size effects tend to overestimate the magnetization at temperatures just above the critical one.
However, for a system with size $N=1000$ the curve is already pretty well overlapped with the theoretical one.

\begin{figure}[htbp]
\includegraphics[scale=0.4]{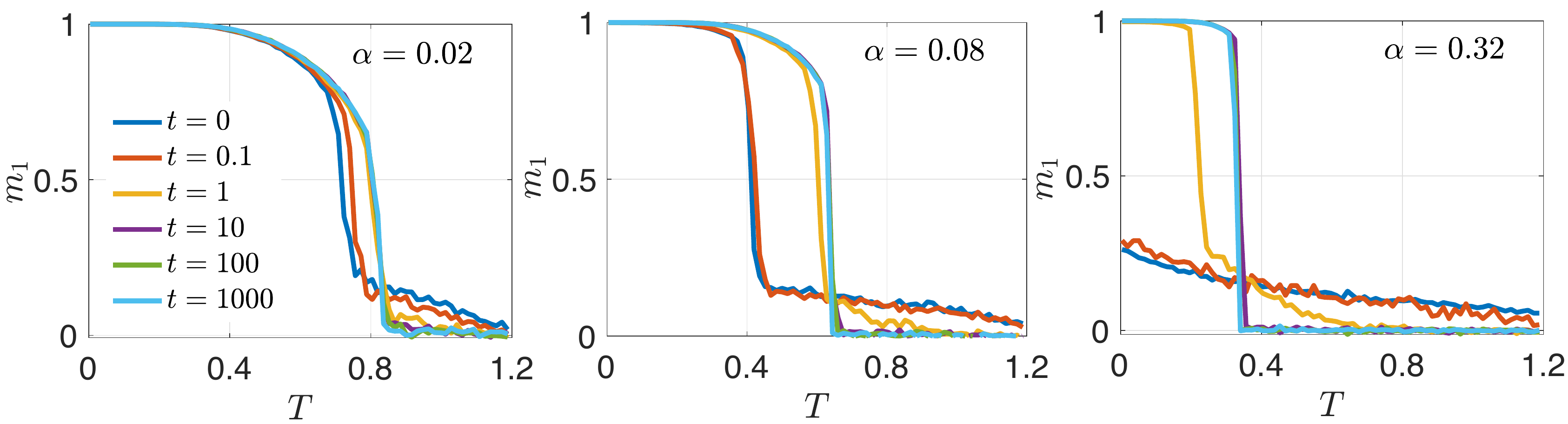}
\caption{{\bfseries Results from Monte Carlo simulations.} These panels report the results from Monte Carlo simulations run for different choices of the parameters $(P,\beta,t)$ and fixing $N=5000$ and $M=10$. More precisely, $1/\beta = T$ ranges from $0$ to $1.2$, $P/N= 0.02$ in the leftmost panel, $P/N= 0.08$ in the middle panel and $P/N= 0.32$ in the rightmost panel. Also, we considered $t=0, 0.1, 1, 10, 100, 1000$, which are depicted in different colors, as explained in the legend.}\label{fig:Ele}
\end{figure}

\begin{figure}[htbp]
\includegraphics[scale=0.425]{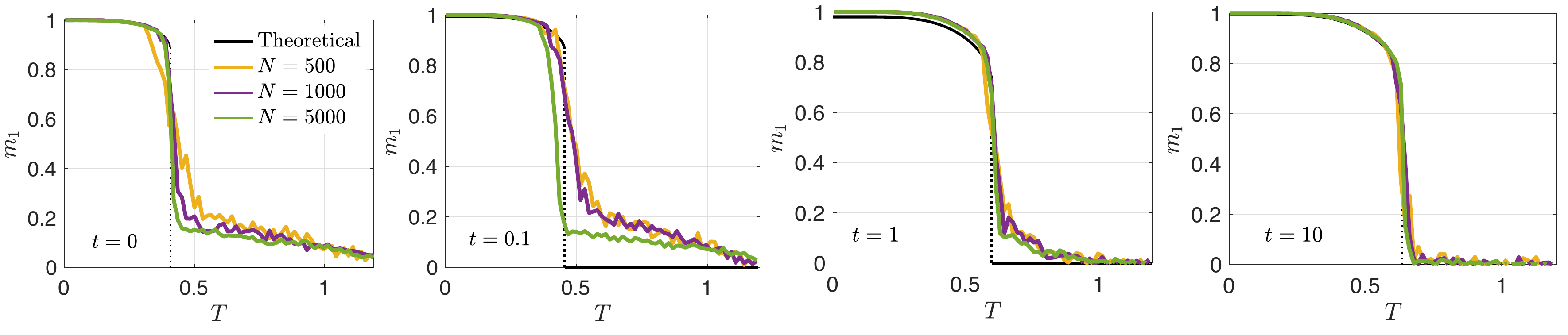}
	\caption{{\bfseries Finite size scaling.} Average values for the Mattis magnetization $m_1$ corresponding to the retrieved pattern $\xi^1$ obtained from numerical simulations for fixed $\alpha=0.08$ and $M=10$. Different sizes ($N=500, 1000, 5000$) are considered and presented in different colors, as explained by the common legend, and are also compared to the theoretical solution reported in Eq.~\ref{eq:sceqs} and obtained in the thermodynamic limit. Moreover, each panel correspond to a different choice of $t$, as reported.}\label{fig:FSS}
\end{figure}

\subsubsection{Fields distributions in retrieved states}
In order to study the internal field distributions characterizing the retrieval mode, we perform extensive MonteCarlo simulations at fixed network size $N$ and for various sleep extents $t$. Since we want to examine the effects of reinforcement and remotion in the retrieval regime, it is convenient to let the network evolve in a point of the tuneable parameters, where retrieval is certainly feasible (namely where pure states dominate the free energy landscape): in the following we will focus on the case $N=1000$ and $P=50$ with a ratio $P/N$ well below the theoretical (Hopfield) critical threshold.\par
A numerical observation is that, as we expect our {\em unlearning$\&$consolidating} algorithm to clean the free energy landscape from metastabilities, we could be able to avoid sophisticated thermalization techniques ({\it e.g.}, simulated annealing \cite{kirkpatrick}). Rather, we aim to check directly if already with rudimental minimizers available for the dynamical update of the neurons, the network is still able to reach a global minimum: these simulations are thus carried with standard Glauber dynamics in the $\beta \to \infty$ limit, with the expressions for the fields acting over the neurons as prescribed by eq. (\ref{new-fields}).
We start the simulations from random initial configurations and simple check that the dynamics ends in a retrieval state. Taking advantage of the mean-field nature of the model, the expression for these fields can be extracted by representing the cost-function (\ref{new-model}) as
\be
H_{N,P}(\sigma|\xi,t) =-\frac{1}{2}\sum_{i=1}^N h_i \sigma_i,
\ee
with the internal fields defined as
\be\label{new-fields}
h_i =\frac{1}{N} \sum_{j=1}^N \sum_{\mu\nu}\xi_{i}^\mu \xi^\nu _j (1+t)(1+tC)_{\mu \nu}^{-1} \sigma_j.
\ee
As stated above, to analyze the internal field configurations, we adopt a standard Glauber dynamics at zero thermal noise level, {\it i.e.} calling $\tau$ the neural update time, with the (parallel) update rule
\be\label{eq:Glauber}
\sigma_i(\tau+1)= \text{sign} [h_i(\tau)],
\ee
so that the simulations stop when all spins are aligned to the internal fields.\footnote{Due to detailed balance, convergence of this kind of algorithm is guaranteed for symmetric synaptic couplings \cite{Amit,Coolen}.}   From the internal field configurations, we estimated numerically the probability density function $P(h)$ and compared it to a standard Gaussian distribution: the results are reported in Fig. \ref{fig:histtall}. Remarkably, the fields distribution $P(h)$ become more and more narrow and peaked as the sleep extent increases. Indeed, the standard deviation $\sigma_{P(h)}$ scaling as a power law of the dreaming time, {\it i.e.} $\sigma_{P(h)} \sim 1/t$, suggesting that dreaming acts, in this picture, as a regularizer in the internal field distributions. The results supporting this picture are shown in Fig. \ref{fig:histtall}.
	\begin{figure}[t]
	\centering
	\begin{minipage}[c]{\textwidth}
		\includegraphics[width=.49\textwidth]{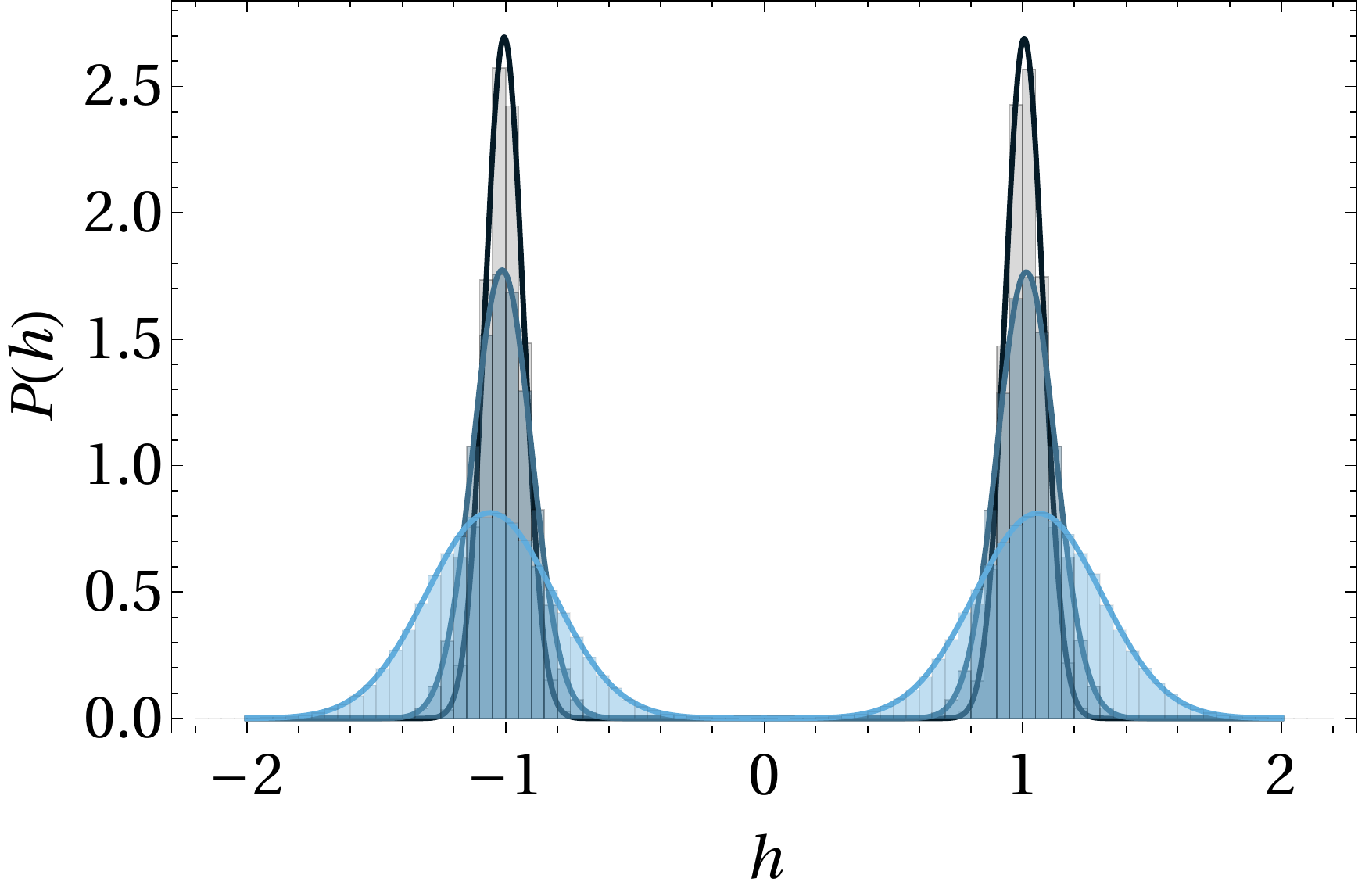}
		\includegraphics[width=.49\textwidth]{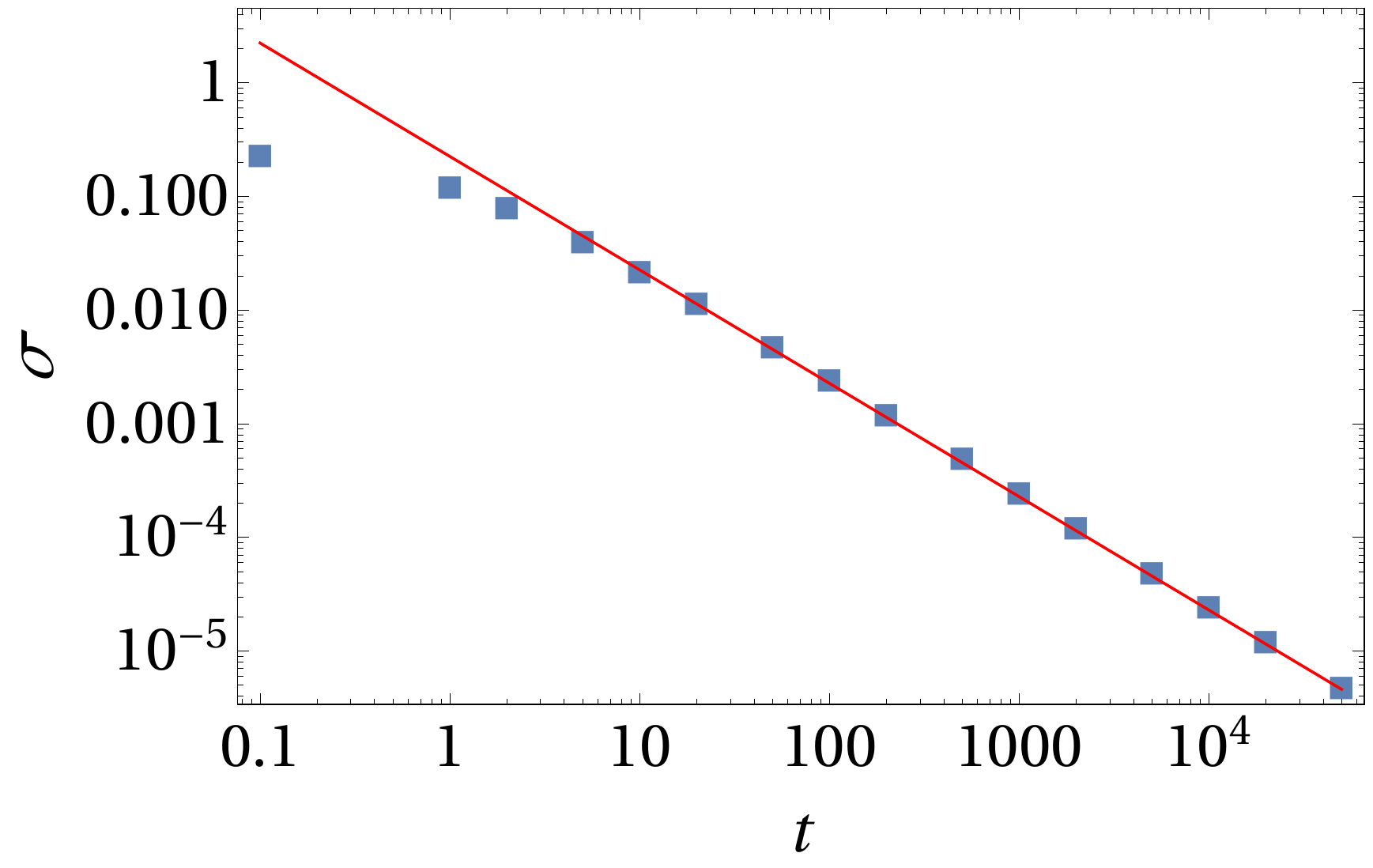}
		\centering
	\end{minipage}%
	\caption{{\bfseries Internal fields probability densities for various unlearning time}. Numerical results (histograms) of the Monte Carlo simulations for the internal fields configuration and comparison with best-fitting Gaussian distributions (smooth curves). The values of the unlearning time here considered are $t=0$ (standard Hopfield case, in light blue), $t=1$ (dark blue) and $t=2$ (light gray). The statistics used in numerical simulations consists in $20$ different stochastic evolutions (with different random initial conditions) and $20$ different realizations of the stored patterns. {\bfseries Dependence of standard deviation of internal fields distribution on the unlearning time.} The plot shows the standard deviation of the (best-fitting) Gaussian distribution of the internal fields configuration as a function of the unlearning time obtained by the previously described Monte Carlo simulations. The results are average on $20$ different stochastic evolutions (with different random initial conditions) and $20$ different realizations of the stored patterns for each unlearning time value. The fit returns a power-law scaling as $\sigma(t) \sim 0.224 \cdot t^{-0.998}$.}\label{fig:histtall}
\end{figure}

\subsubsection{Retrieval frequency for noisy inputs: on the attraction basins}
As a natural successive step, we need to (partially) reintroduce the noise in the network and use it to analyze the depth of the free energy pure minima: the underlying idea is to present some noisy inputs to the network (at various noise intensities $p$) and check the proper signal reconstruction. Otherwise stated, check if, once supplied a noisy input, the network is still able to find its path to the related global free energy minimum accounting for the correct pattern.
\newline
Also in this case, the parameters are fixed to $N=1000$ and $P=50$.\footnote{We checked that analogous results hold also for (randomly selected) different configurations (whose results we do not report).} The procedure we adopted is standard: we prepare the network $\{ \sigma_i \}_{i=1,...,N}$ aligning it to the first pattern $\xi^1$, then we flip each neuron ($\sigma_i\rightarrow -\sigma_i$) with probability $p$. In this way, we construct the initial state of the network. Then, we let the network evolve according to the classical zero-noise Glauber dynamics (see eq. (\ref{eq:Glauber})) and count how many times the signal is properly reconstructed (in other words, we check that the overlap of the network  state with the candidate pattern $\xi^1$, {\it i.e.} $m_1$ is the maximal Mattis magnetization). The results are plotted in Fig. \ref{fig:retr}, and show that our algorithm has the effect of making the basins of the pure attractors more stable with the sleeping session:  this is intuitively in agreement with the observation that - increasing the sleep extent - the retrieval region becomes larger (w.r.t. the Hopfield reference).

\begin{figure}[htbp]
	\centering
	\begin{minipage}[c]{.7\textwidth}
		\includegraphics[width=\textwidth]{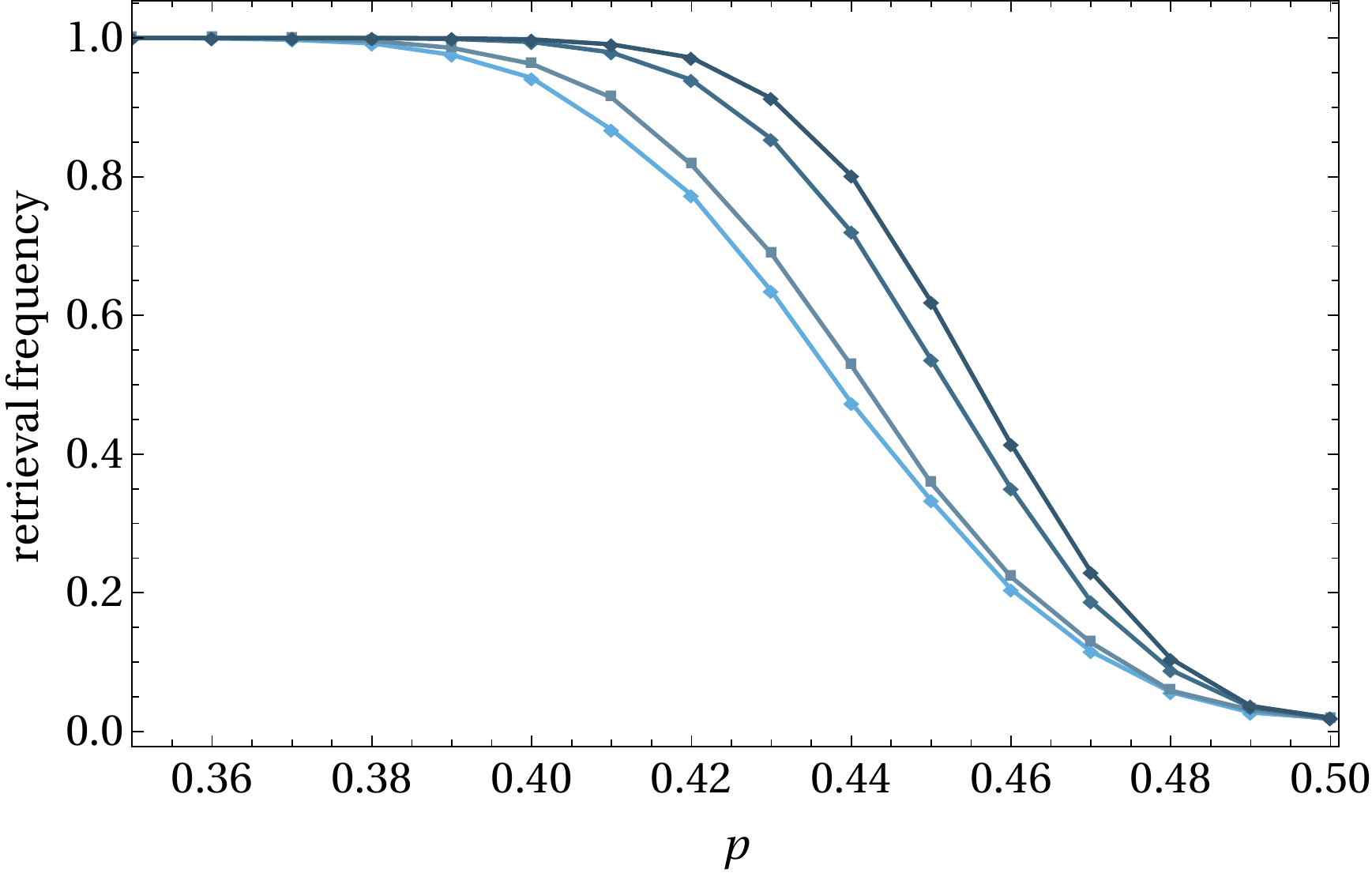}
		\centering
	\end{minipage}%
	\caption{{\bfseries Analysis of attraction basins.} The plots shows the retrieval frequency as a function of the spin-flip probability for $t=0$, $0.1$, $1$ and $1000$ (from the left to the right). These results are obtained with 200 different stochastic evolutions for each of the 200 pattern realizations.}\label{fig:retr}
\end{figure}

\newpage

\section{{\em Unlearning$\&$Consolidating}: Focusing on Synapses}\label{sezione-quattro}
\subsection{Time evolution of the synaptic matrix}

The Hebbian paradigma can be interpreted as the adiabatic collection of memories sequentially learnt. Along the same line, we show that the coupling (\ref{new-model}) encoding for reinforcement {and} removal can be seen as the result of a sequential synaptic updating. To this goal it is convenient to first look at the evolution of the coupling $J$ as a continuous dynamic process. Later, we will show how to make it discrete and suitable for an iterative implementation.

\subsubsection{The continuous algorithm}
By construction, the interpolating coupling matrix
\be
J_{ij} (t)= \frac{1}{N}\sum_{\mu\nu} \xi^\mu _i \xi^\nu_j \left(\frac{1+t}{1+t C}\right)_{\mu\nu},
\ee
has as limiting cases $J(0)= J$ and $J(\infty)=J^{p}$. Exploiting the identity
\be
\frac{1}{N}\sum_{\mu\nu}\xi^\mu _i \xi^\nu_j (C^n)_{\mu\nu}= \sum_k J(0)_{ik}(J(0)^n)_{kj},
\ee
we can recast the time-dependent matrix $J(t)$ as
\be\label{eq:Jform}
J(t)=(1+t) J(0) (1+t J(0))^{-1}.
\ee
Upon differentiating \eqref{eq:Jform} with respect to $t$, we end with the evolution equation
\be
\dot J = \frac{1}{1+t} (J-J^2).
\ee
Comparing this matrix ODE with standard unlearning process in the Literature \cite{Plakhov,Semenov1,Semenov2}, we notice two main difference. First, we have a non-trivial dependence on the sleep time through the prefactor $(t+1)^{-1}$. Second, there are two contributions in the evolution equation, associated to different scaling (respectively, linear and quadratic in $J$) and opposed signs (mimicing the two opposite features of the algorithm, {\it i.e.} consolidation and remotion). %
\par\medskip
\subsubsection{The discrete algorithm}
To go to the discrete picture, we have to (re)introduce a tunable parameter $\epsilon$ (the {\it unlearning strength} or equivalently {\em effectiveness of sleep}) defining a temporal scale in which remotion and consolidation are effective. Therefore, we perform the following replacements:
\be
\begin{split}
dt &\rightarrow \epsilon,\\
t&\rightarrow k\epsilon,\\
\dot J&\rightarrow \epsilon^{-1}[J(k+1)-J(k)],
\end{split}
\ee
where $k$ labels the number of sleeping sessions. Thus, at a discrete level, the {\em reinforcement$\&$remotion} procedure is provided by the following rule:
\be\label{eq:unlearningrule}
J(k+1)= J(k)+\frac{\epsilon}{1+\epsilon k} [J(k)-J(k)^2].
\ee

\begin{figure}[t!]
	\begin{minipage}[c]{.49\textwidth}
		\centering
		\includegraphics[width=\textwidth]{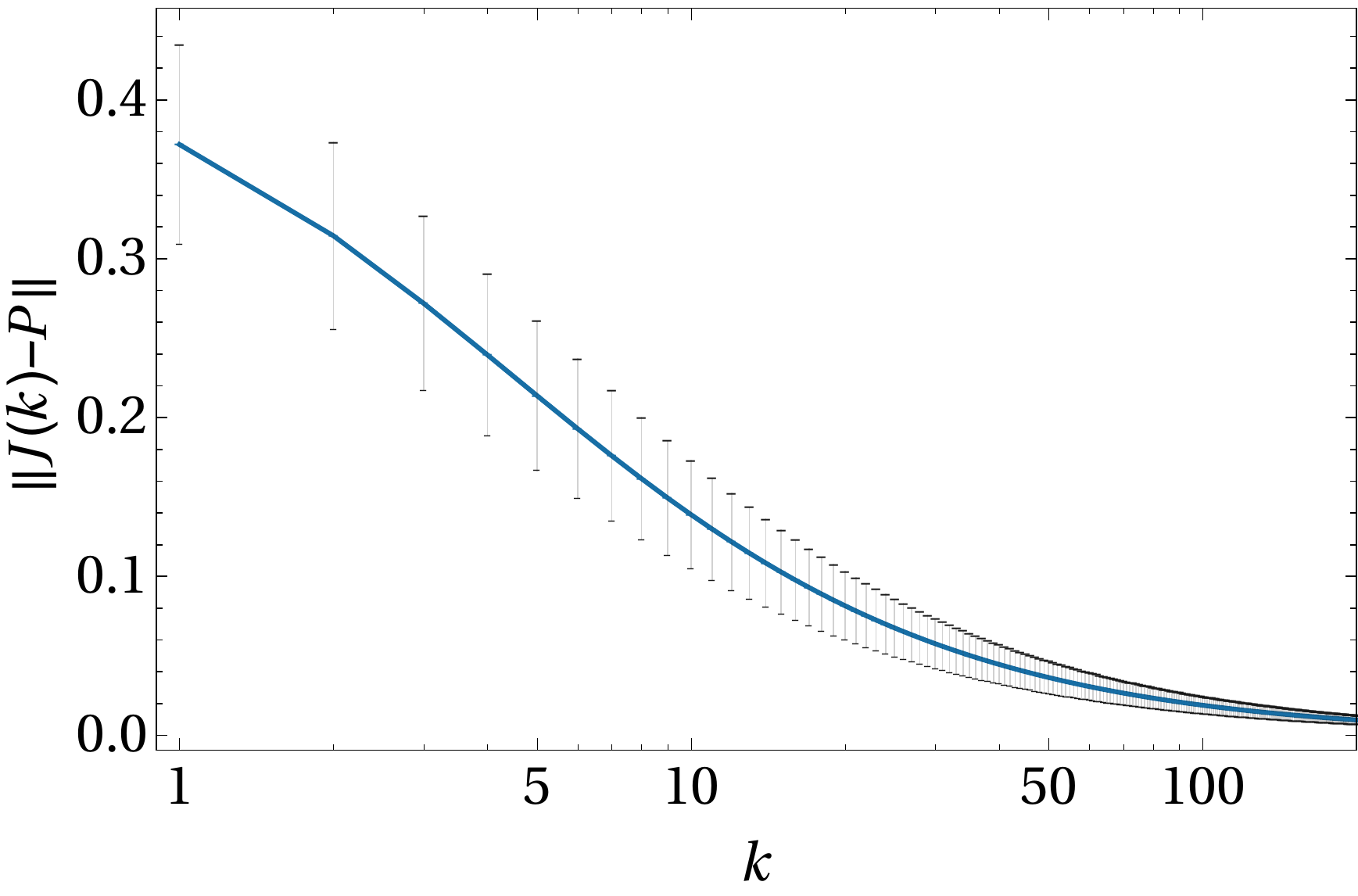}
	\end{minipage}%
	\hspace{-0mm}%
	\begin{minipage}[c]{.49\textwidth}
		\centering
		\includegraphics[width=\textwidth]{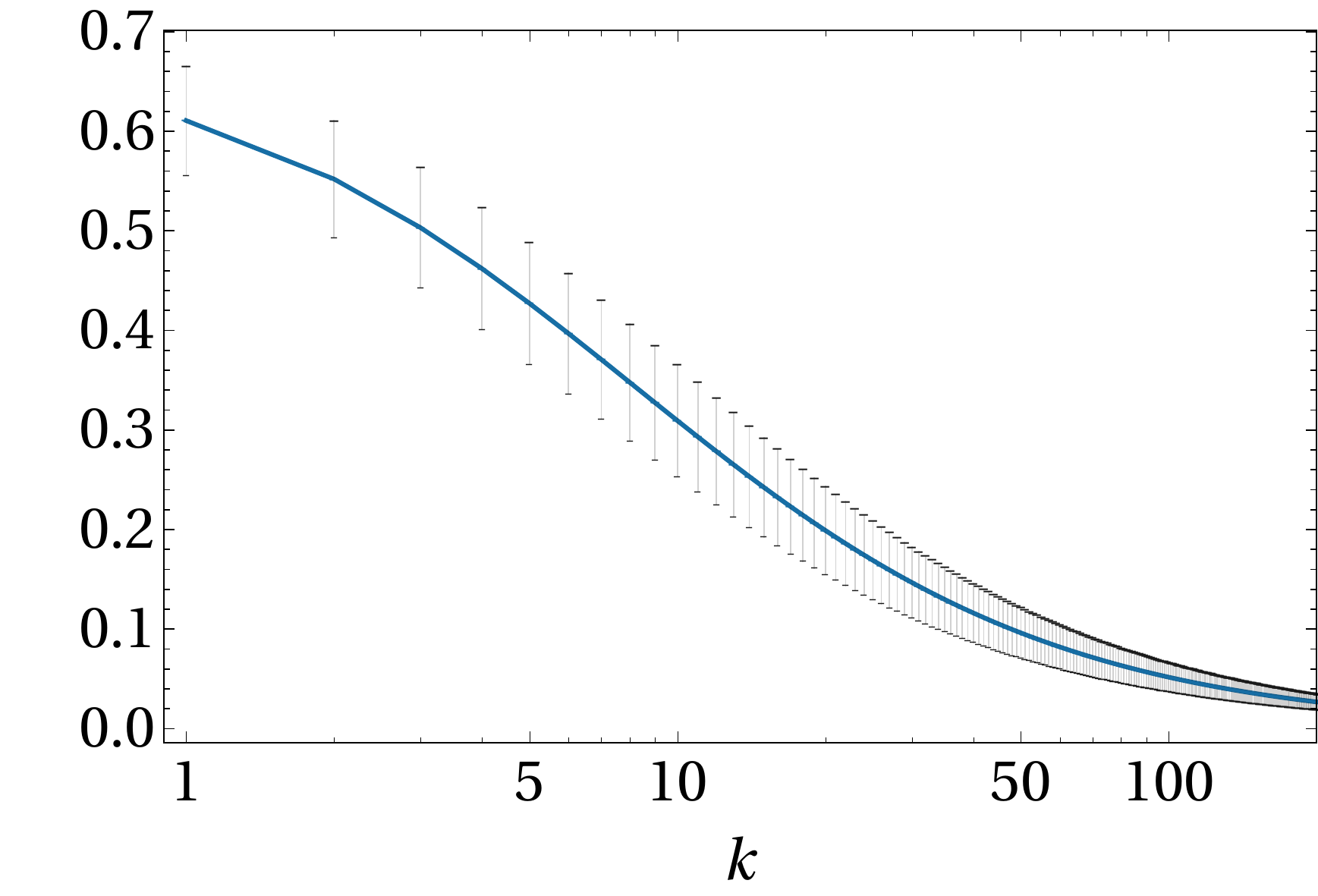}
	\end{minipage}
	\begin{minipage}[c]{.49\textwidth}
		\centering
		{\includegraphics[width=\textwidth]{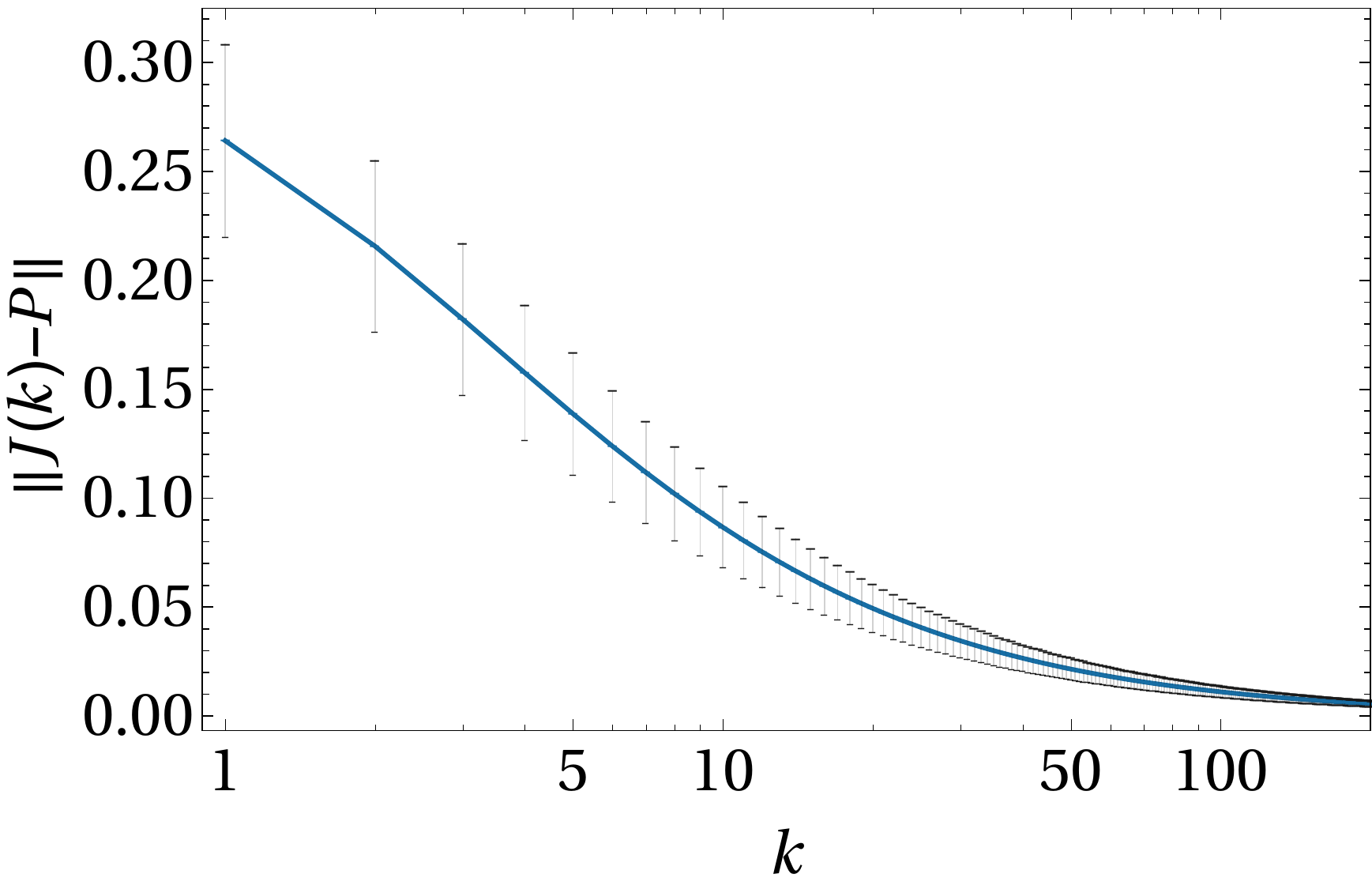}}
			\end{minipage}%
	\hspace{-1.2mm}
	\begin{minipage}[c]{.49\textwidth}
		\centering
		{\includegraphics[width=\textwidth]{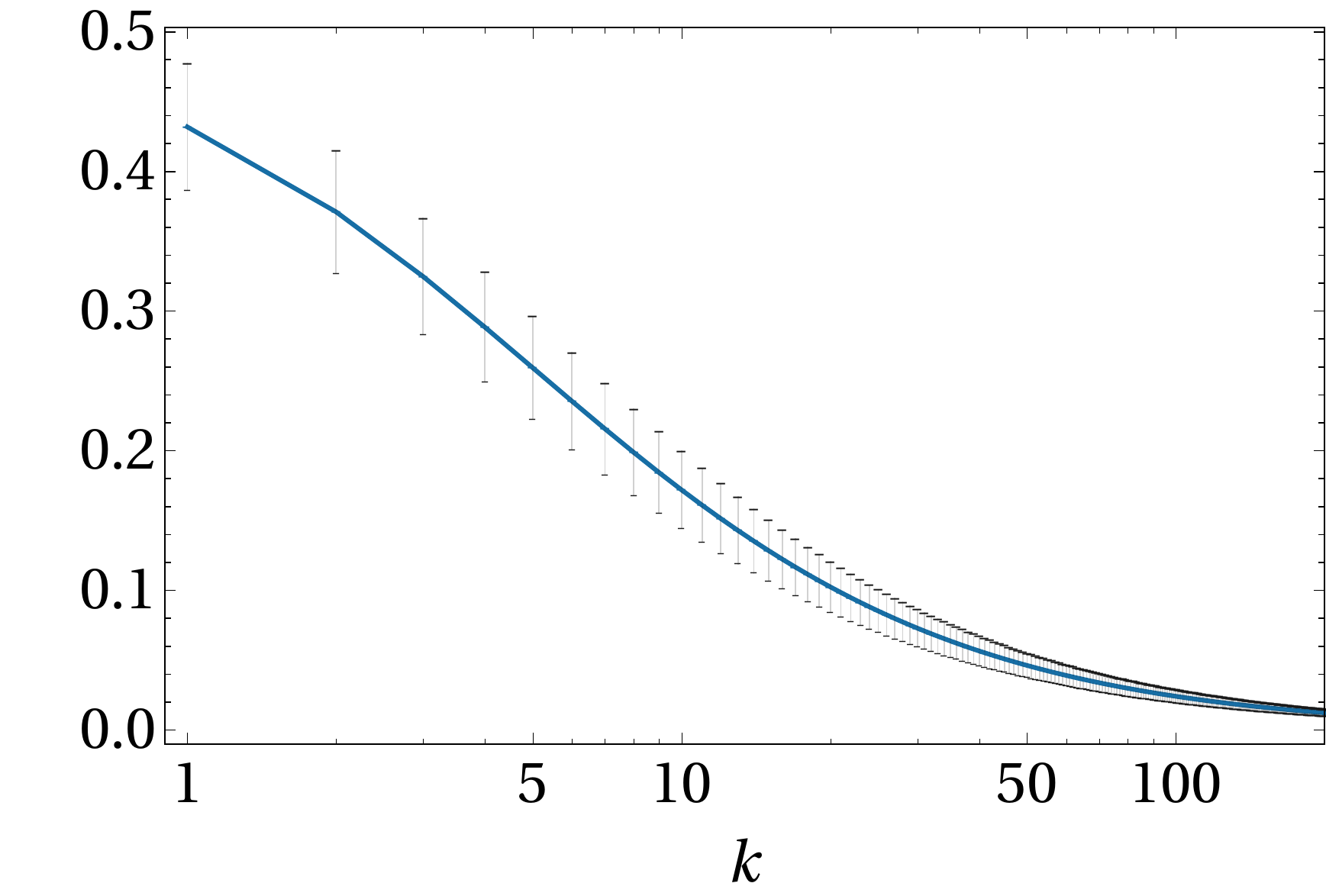}}
	\end{minipage}
	\caption{{\bfseries Unlearning procedures for various networks}. The four plots show the convergence (in the the operator norm) of the coupling matrix with the unlearning procedure \eqref{eq:unlearningrule} for various network parameters: first line - $N=64$ ($P=8$ and $P=16$); second line - $N=128$ ($P=8$ and $P=16$). The temporal window shown in the plots is limited to for the first $200$ cycles. The results are the average over 500 different realizations of the patterns. The parameter $\epsilon$ is fixed to $0.5$.} \label{fig:convergence}
\end{figure}
We stress that this procedure can naturally be interpreted as an {\it adaptive} unlearning $\&$ consolidation scheme, as the {\it effective} sleep strength ({\it i.e.} the coefficient of the $J$-dependent corrections in eq. \eqref{eq:unlearningrule}) does depend on the sleep session $k$. With this prescription, the coupling matrix converges to the projection matrix in the limit of infinite dreams, as is clear from Fig.~\ref{fig:convergence}. To prove this, it is convenient to write the coupling matrix in the form
\be
J_{ij}(k) = \frac1N\sum_{\mu\nu}\xi^\mu _i \xi^\nu _j G_{\mu\nu}(k).
\ee
Then, the above unlearning$\&$consolidation rule can be recast as
\be\label{eq:Gmatrix}
G(k+1)= \Big(1+\frac{\epsilon}{1+\epsilon k}\Big)G(k)- \frac{\epsilon}{1+\epsilon k} G(k) C G(k),
\ee
with the initial condition $G(0)=\mathbb I$. The analytical proof of the convergence of this algorithm is reported in Appendix. \ref{app:convergence}. An important point we would like to stress is that the critical value of the unlearning strength ensuring the convergence can be sharply estimated as (see Appendix \ref{app:convergence}, Corollary \ref{cor:ollario})
\be\label{eq:critical}
\epsilon_c= \frac{1}{\lVert C \rVert -1}.
\ee
This equality is very instructive, since it states that the critical strength for the synaptic update is fixed by the magnitude of the patterns' correlations: the stronger the correlations, the longer the amount of time required to the sleep for optimizing the network's free energy landscape.
\par
We also notice that the singular case $\lVert C \rVert =1$ is excluded, since it only happens if all the eigenvalues of $C$ are 1, meaning that $C=\mathbb I$ (or equivalently, all patterns are uncorrelated). In the latter case, no unlearning is needed, since the starting point is exactly the solution of the above recurrence relation.\footnote{We stress, however, for identally distributed random boolean patterns for finite $N$, the correlation is always presented with a rough estimation $\sim 1/\sqrt{N}$.}\par
It is interesting to compare these results with the unlearning procedure analyzed in \cite{Semenov1}, for which the critical unlearning strength is given by
\be\label{eq:criticalSem}
\epsilon_c = (N \lVert J(0)\rVert )^{-1},
\ee
where $J(0)$ is the Hopfield coupling matrix. By inspecting at Figure \ref{fig:criticalS}, it is clear that, in our case, the critical unlearning strength is higher than the usual one \eqref{eq:criticalSem} by many order of magnitudes, and (in the low storage case, {\it i.e.} $P \ll N$) it is independent on $N$. Another important difference between the two algorithms is that, while for \eqref{eq:criticalSem} and fixed $P$ the critical strength fastly decreases as $N$ grows, in our case it slowly increases for higher network size. Thus, our method appears to be more stable with respect to the network size.

\begin{figure}[htbp]
	\centering
	\begin{minipage}[c]{.7\textwidth}
		\includegraphics[width=\textwidth]{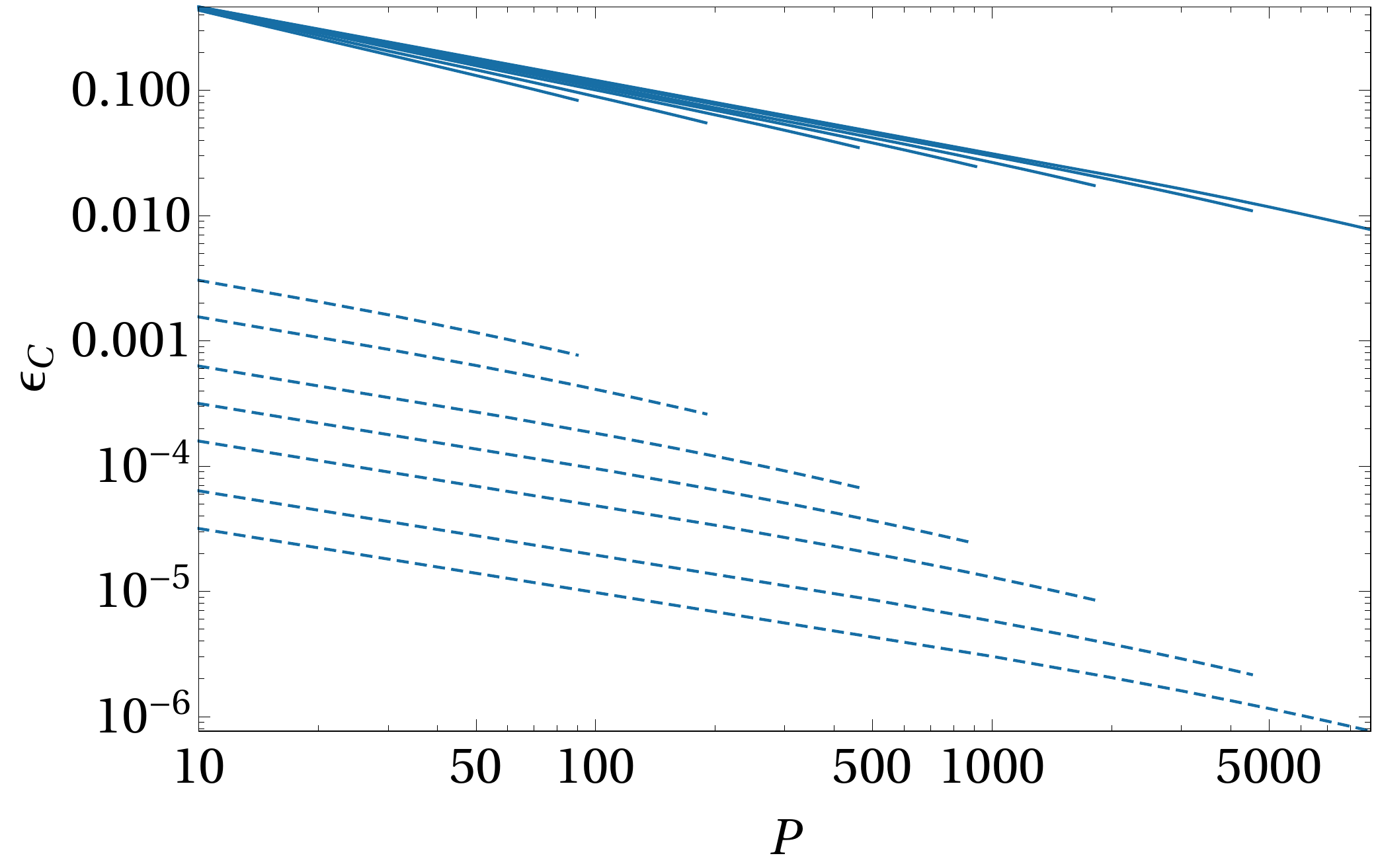}
		\centering
	\end{minipage}%
	\caption{{\bfseries Dependence of the critical unlearning strength versus $P$ and $N$.} The plot shows the values of the critical unlearning strength as functions of the number of neurons $N$ and the number of stored patterns $P$. Each curve corresponds to a fixed value of $N=100$, $200$, $500$, $1000$, $2000$, $5000$ and $10000$. The solid blue lines are the critical strength \ref{eq:critical} of our unlearning procedure \ref{eq:unlearningrule}. The dashed blue lines are instead associated to the critical strength \eqref{eq:criticalSem}. The results are averaged on $50$ different realizations of the patterns (the statistical errors are not reported since they are too small to be visualized in the log-log scale). }\label{fig:criticalS}
\end{figure}

\section{Conclusions} 
Inspired by information optimization during sleep episodes in mammal's brains, in this paper we study Hebbian unlearning with reinforcement, namely we discuss how the Hopfield model (where the bare Hebb prescription is forecasted) can be generalized to better optimize its resources, namely, in order to have the most possible robustness w.r.t. (fast/thermal) noise and the larger  possible capacity (i.e. the ratio among the stored patterns and the neurons available to handle them).
\newline
Oversimplifying, during human's sleep, two main types of dreaming  alternate, namely the slow-wave and the random-eye-movement phases, with two coupled -but different- purposes: while they share the final goal of achieving best possible optimization of information storage, the former contributes to the scope by consolidating important memories (that we match with the patterns in the AI counterpart played by associative neural networks), the latter instead gets rid of the -by far more abundant- unimportant memories (that we match with spurious/mixture states in the AI counterpart played by associative neural networks).
\newline
To account for both these features at once in AI too, we proposed a novel {\em unlearning$\&$consolidating} algorithm that we ideally use to stylize a dream (with the same spirit by which a neuron is reduced to a Boolean variable in mathematical modeling) and we tested it on the standard reference provided by the Hopfield model. We stress that, while in Hopfield networks  just pairwise correlations are stored, the  present theory can be applied in a much broader generality, e.g. for instance extending the cost-function to account for P-spin higher-order contributions \cite{Albert,Dimitry,Metha} for deep nets too.
\newline
Our algorithm has solely one novel parameter, accounting for the time the network spent in dreaming, and we studied how both the neural performances as well as the synaptic couplings evolve as this parameter is tuned.
As far as the neurons are concerned, as a result of our procedure, at the amount of dreams increases we obtain a significant improvement in the critical capacity that, at first (mainly thanks to discarding spurious states), in the zero fast noise limit, increases from $\alpha_c \sim 0.14$ to $\alpha_c \sim 1$ (that is the maximal capacity if the network is equipped with symmetric couplings, as prescribed by the Gardner theory \cite{Gardner,VanHemmen}), further (mainly due to the reinforcement term), pure memories remain stable even against high level of (fast) noise. Indeed we inspected how the fields acting on the neurons gets affected by the dreams and, as the dreaming time increases, the fields get better and better peaked over pattern's entries -getting rid of the noise- and, remarkably, their standard deviations $\sigma$ have a power-law scaling with the dreaming time, i.e. $\sigma \propto t^{-1}$. It is also worth pointing out that network performances increase in a high non-linear way with the dreaming time (such that with a few cycles a massive optimization has already been achieved and there is no longer need to reach un-physical epochs).
\newline
This gets crystal clear when focusing on the synapses as, already an elementary glance at their dynamical evolution, suggests that the -starting with standard pairwise Hopfield- the dynamics forces the Hebbian kernel to match the projection matrix, close to the scenario pictured by Kanter and Sompolinsky \cite{KanterSompo}. We confirmed this statement both analytically and numerically and we found a sharp estimate for the optimal {\em dreaming rate} -the analogous of a learning or unlearning rate(s) in existing Literature.
\newline
Finally we aim to notice that there is also another important reason to investigate these improvements over the standard scenario in Hebbian machines: there is a one-to-one correspondence among Hopfield networks and restricted Boltzmann machines \cite{Agliari-PRL1,BarraEquivalenceRBMeAHN} thus, as Boltzmann machines are the building blocks in modern Deep Learning architectures \cite{DL1,Hinton1}, increasing efficiency in the former may imply progress even in the latter, and ultimately in Deep Learninsg (as we know that modern machines trained with deep learning actually  do dream of electric sheeps \cite{Guardian}\footnote{Further, this interpretation of sleep$\&$dream raises as a stand-alone alternative against the Freudian psychoanalysis, as brilliantly pointed out by Christos in \cite{Cristo}, but in this manuscript this point will not be deepened.}).

\section*{Acknowledgements}

A.F. and A.B. acknowledge Salento University, MIUR (through basic funding to the Italian research) and INFN for partial support.\\
A.B. also acknowledges the grant {\em Rete Match: Progetto Pythagoras (CUP:J48C17000250006)}.\\
E.A. acknowledges the grant {\em Progetto Ateneo (RG11715C7CC31E3D)} from Sapienza University of Rome.\\
E.A., A.B. and A.F. are grateful to GNFM-INdAM for partial financial support.

\appendix

\section{Calculations to obtain the replica symmetric solution}\label{app:replica}
In this Appendix, we report in some detail the replica trick\footnote{A solid mathematical ground for the replica trick in the Sherrington-Kirkpatrick model for spin-glasses is already available (see e.g., \cite{BarraGuerraMingione}), while in the Hopfield model for neural networks this is only partially available (see e.g., \cite{TirozziRev}).} calculations necessary to get an explicit expression, in terms of the order parameters, of the (replica-symmetric) free energy of the model. We start with the (quenched average of the) replicated partition function \eqref{eq:boltzmann}, which we rewrite here as
\begingroup\makeatletter\def\f@size{9.5}\check@mathfonts
\be
\begin{split}
\mathbb E Z_{N,P}(\sigma|\xi,t)^n   =&\sum_{\sigma^1}\dots\sum_{\sigma^n}\int\Big(\prod_{\alpha}Dz_1^\alpha\Big)\Big(\prod_{i\alpha}D\phi_i^\alpha\Big)\exp\Big[\sqrt{\frac{\beta(t+1)}{N}}\sum_{i \alpha} z_1^\alpha \xi^1 _i \left(\sigma_i ^\alpha +i \sqrt{\frac{t}{\beta (t+1)}}\phi_i ^\alpha\right)\Big]\cdot\\&
\cdot \int \Big(\prod_{\alpha}\prod_{\mu \ge 2}Dz_1^\alpha\Big)\mathbb E' \exp\Big[\sqrt{\frac{\beta(t+1)}{N}}\sum_{i \alpha}\sum_{\mu\ge 2} z_\mu^\alpha \xi^\mu _i \Big(\sigma_i ^\alpha +i \sqrt{\frac{t}{\beta (t+1)}}\phi_i ^\alpha\Big)\Big].
\end{split}
\ee
\endgroup
Note that we separated the signal term (associated to pattern $\xi^1$, meant to be retrieved) and the slow noise (constituted by all the other not-retrieved patterns, whose random similarities with $\xi^1$ -i.e. the spurious correlations this paper is due to- lie at core-genesis of such a slow noise). The exponential with noisy contributions in the second line can be easily rewritten (neglecting sub-leading contributions in the large $N$ limit) as
\begingroup\makeatletter\def\f@size{9.2}\check@mathfonts
\begin{equation*}\begin{split}
	&\mathbb E \exp\Bigg[\sqrt{\frac{\beta(t+1)}{N}}\sum_{i \alpha}\sum_{\mu\ge 2} z_\mu^\alpha \xi^\mu _i \Big(\sigma_i ^\alpha +i \sqrt{\frac{t}{\beta (t+1)}}\phi_i ^\alpha\Big)\Bigg]=\int\prod_{\alpha\beta}dq_{\alpha\beta}\frac{Ndp_{\alpha\beta}}{2\pi}\cdot\\&\cdot\exp\Big[i N\sum_{\alpha\beta}p_{\alpha\beta}q_{\alpha\beta}+\frac{\beta(t+1)}{2}\sum_{\alpha\beta}\sum_{\mu \ge 2}z_{\mu}^\alpha z_\mu ^\beta q_{\alpha\beta}
	-i\sum_{\alpha\beta}p_{\alpha\beta }\Big(\sigma_i ^\alpha +i \sqrt{\frac{t}{\beta (t+1)}}\phi_i ^\alpha\Big)\Big(\sigma_i ^\beta +i \sqrt{\frac{t}{\beta (t+1)}}\phi_i ^\beta\Big)\Big],
\end{split}\end{equation*}
\endgroup
where we imposed the definition of overlap \eqref{eq:overlap} through the insertion of a Dirac delta (in its Fourier representation, as standard \cite{Coolen}).
We can then perform the Gaussian integration over the order parameters $z_\mu^\alpha$ which are not associated to the retrieved pattern, so to obtain
\begingroup\makeatletter\def\f@size{8.7}\check@mathfonts
\be
\begin{split}
	&\mathbb E Z_{N,P}(\sigma|\xi,t)^n=\sum_{\sigma^1}\dots\sum_{\sigma^n}\int\Big(\prod_{\alpha}Dz_1^\alpha\Big)\Big(\prod_{i\alpha}D\phi_i^\alpha\Big)\Big(\prod_{\alpha\beta}dq_{\alpha\beta}\frac{Ndp_{\alpha\beta}}{2\pi}\Big)
	\exp\Big[i N\sum_{\alpha\beta}p_{\alpha\beta}q_{\alpha\beta}-\frac{p}{2}\log\det (\mathbb I -\beta (1+t) \hat{q})\\&+\sqrt{\frac{\beta(t+1)}{N}}\sum_{i \alpha} z_1^\alpha \xi^1 _i \Big(\sigma_i ^\alpha +i \sqrt{\frac{t}{\beta (t+1)}}\phi_i ^\alpha\Big)
	-i\sum_{\alpha\beta}p_{\alpha\beta }\Big(\sigma_i ^\alpha +i \sqrt{\frac{t}{\beta (t+1)}}\phi_i ^\alpha\Big)\Big(\sigma_i ^\beta +i \sqrt{\frac{t}{\beta (t+1)}}\phi_i ^\beta\Big)\Big].
\end{split}
\ee
\endgroup
We replace the order parameter $z_1^\alpha$ with the corresponding (replicated) Mattis magnetization by using the relation $z_1 = \sqrt{{\beta N}({1+t})^{-1}}m_1^\alpha$ (see also \cite{DotsenkoDorotheyev,DotsenkoTirozzi}). We also make the convenient redefinition of the conjugated overlap $p_{\alpha\beta}\rightarrow i \frac{\alpha \beta^2}{2}p_{\alpha\beta}.$
After some trivial rearrangements, we get
\begingroup\makeatletter\def\f@size{9.5}\check@mathfonts
\be
\begin{split}
	&\mathbb E Z_{N,P}(\sigma|\xi,t)^n=\int\Big(\prod_{\alpha}\sqrt{\frac{\beta N}{2\pi(1+t)}}dm_1^\alpha\Big)\Big(\prod_{\alpha\beta}dq_{\alpha\beta}\frac{iN\alpha\beta^2dp_{\alpha\beta}}{4\pi}\Big)\exp\Big\{-\frac{\beta N}{2}\sum_{\alpha}\frac{{m_1 ^\alpha}^2}{1+t}\\&-\frac{N \alpha\beta^2}{2} \sum_{\alpha\beta}p_{\alpha\beta}q_{\alpha\beta}-\frac{\alpha N}{2}\log\det (\mathbb I -\beta (1+t) \hat{q})
	+\log\sum_{\sigma^1}\dots\sum_{\sigma^n}\int\Big(\prod_{i\alpha}D\phi_i^\alpha\Big)\cdot \\&\cdot \exp\Big[\beta\sum_{i \alpha} m_1^\alpha \xi^1 _i \Big(\sigma_i ^\alpha +i \sqrt{\frac{t}{\beta (t+1)}}\phi_i ^\alpha\Big)+\frac{\alpha\beta^2}{2}\sum_{\alpha\beta}p_{\alpha\beta }\Big(\sigma_i ^\alpha +i \sqrt{\frac{t}{\beta (t+1)}}\phi_i ^\alpha\Big)\Big(\sigma_i ^\beta +i \sqrt{\frac{t}{\beta (t+1)}}\phi_i ^\beta\Big)\Big]\Big\}.
\end{split}
\ee
\endgroup
The last line can be easily handled and its terms rearranged in order to remove the site index ({\it i.e.} the subscript $i$) from the spins $\sigma$ and the auxiliary Gaussian fields $\phi$. Moreover, since in the thermodynamic limit the ``hergodic'' equality
\be
\log \prod _i f(\xi^1_i)=\sum_i \log f(\xi^1_i)= N \mathbb E f(\xi),
\ee
holds \cite{Amit,Coolen}, we can easily represent the replicated partition function in the form
\be
	\mathbb E Z_{N,P}(\sigma|\xi,t)^n=\int
	d\mu (m_1 ^\alpha, q_{\alpha\beta},p_{\alpha\beta})e^{-\beta N n A},
\ee
$d\mu$ being the measure over all the order parameters and
\be
\begin{split}
	A(\alpha,\beta,t) &=\frac{1}{2n(1+t)}\sum_\alpha {m_1 ^\alpha}^2+\frac{\alpha \beta}{2n}\sum_{\alpha\beta}p_{\alpha\beta}q_{\alpha\beta}+\frac{\alpha}{2n\beta}\log\det(\mathbb I-\beta(1+t)\hat q)\\&
	-\frac{1}{n\beta}\mathbb E\log \sum_\sigma \int\Big(\prod_\alpha D\phi^\alpha\Big)\exp\Big[\beta \sum_\alpha m_1^\alpha \xi^1\Big(\sigma^\alpha+i\sqrt{\frac{t}{\beta(1+t)}}\phi^\alpha\Big)+\\&
	+\frac{\alpha \beta^2}{2}\sum_{\alpha\beta}p_{\alpha\beta}\Big(\sigma^\alpha+i\sqrt{\frac{t}{\beta(1+t)}}\phi^\alpha\Big)\Big(\sigma^\beta+i\sqrt{\frac{t}{\beta(1+t)}}\phi^\beta\Big)\Big].
\end{split}
\ee
being the general free-energy of the model, see \eqref{eq:fgeneral}.
\newline
Imposing the replica symmetric ansatz and recalling the definition of $\Delta$, we can compute the replica-symmetric free energy $A(\alpha,\beta,t)$ term by term:
\begingroup\makeatletter\def\f@size{9.5}\check@mathfonts
\begin{subequations}
\begin{align*}
\circ\quad &\frac{1}{2n}\sum_{\alpha}\frac{{m_1 ^\alpha}^2}{1+t}=\frac{m_1 ^2}{2(1+t)},\\
\circ\quad&\frac{\alpha\beta}{2n}\sum_{ \alpha\beta}p_{\alpha\beta}q_{\alpha\beta}= \frac{(\Delta-1)(1+t)}{2t}Q+\frac{\alpha\beta}{2}p(Q-q),\\
\circ\quad&\frac{\alpha}{2n\beta}\log\det[\mathbb I -\beta(1+t)\hat q ]=\frac{\alpha\beta}{2}\left(\log [1-\beta(1+t) (Q-q)]-\frac{q\beta(1+t)}{1-\beta(1+t)(Q-q)}\right)+\mathcal O(n),\\
\circ\quad&\frac{1}{n\beta}\mathbb E \log \sum_{ \sigma}\int\Big(\prod_{\alpha}\phi^\alpha\Big)\Big[\dots\Big]=\frac{1+t}{2t}\frac{\Delta-1}{\Delta}-\frac{1}{2\beta}\log \Delta -\frac{ \alpha pt}{2\Delta (1+t)}-\frac{ t}{2\Delta(1+t)}m_1^2+\frac{1}{\beta}\log 2\\&+\frac{1}{\beta} \int Dx \, \log \cosh \Big[\frac{\beta}{\Delta}(m_1 +\sqrt{\alpha p}x)\Big].
\end{align*}
\end{subequations}
\endgroup
Putting all pieces together and taking the limit $n\rightarrow 0$, after some rearrangements we arrive at the free energy expression \eqref{eq:frsa}.

\section{Convergence of the discrete algorithm: the analytical proof}\label{app:convergence}
In this Appendix, we prove the convergence of the unlearning rule \eqref{eq:unlearningrule} toward the inverse correlation matrix $C^{-1}$. In doing this, we will follow a procedure which is very close to the route paved in \cite{Plakhov}.
\newline
The norm in the matrix vector space we used is the {\it operator norm}, which means that
\be
\lVert A\rVert = \sqrt{\max\{a\vert a \in \sigma(A^{\text T}A)\}},
\ee
where $\sigma(A^{\text T}A)$ is the spectrum of the matrix $A^{\text T}A$. Note that, since we will deal only with symmetric and positive-definite matrices, this definition reduces to $\lVert A\rVert = \max \{a\vert a \in \sigma(A)\}$. In what follows, we will often use the notations
\be
q_k = \frac{1+\epsilon k}{1+\epsilon(k+1)},\quad p_k = \frac{\epsilon}{1+\epsilon(k+1)}.
\ee

\begin{Proposition}
The norm of correlation matrix is greater than one: $\lVert C \rVert\ge 1$.
\end{Proposition}
\begin{proof}
	By definition, the diagonal entries of the correlation matrix are all equal to 1, so that
	\be
	\text{Tr} \,C= p.
	\ee
	Moreover, the correlation matrix $C$ is symmetric and positive-definite. Then, all eigenvalues $\gamma_{\mu}$ are clearly positive, and we have
	\be
	\sum_{\mu}\gamma_{\mu}=p.
	\ee
	Since the number of the eigenvalues is precisely $p$, it is impossible to saturate the trace equality with $\gamma_\mu <1$ for all $\mu$. Since the largest eigenvalue is equal to the matrix norm, it follows that $\lVert C\rVert\ge 1$.	
\end{proof}
\begin{Proposition}
	The matrix $G(k)$ commutes with $C$ for all $k$.
\end{Proposition}
\begin{proof}
	We define two new matrix types my multipling $G(k)$ on the right and on the left with $C$, {\it i.e.} $T^{(1)}(k)=G(k)C $ and $T^{(2)}(k)=C G(k)$. Since $G(0)=\mathbb I$, then $T^{(1,2)}=C$. It's easy to see from \eqref{eq:Gmatrix} that both $T^{(1)}$ and $T^{(2)}$ satisfy the same recursion relation, and since they have the same initial condition, it follows that $T^{(1)}(k)= T^{(2)}(k)$ for all $k$, which means that $G(k)C=CG(k)$ proving the statement.
\end{proof}
\begin{Proposition}
	The matrices $G(k)$ are invertible for all $k\ge 0$ and $\epsilon < \epsilon_c$ for some $\epsilon_c$.
\end{Proposition}
\begin{proof}
	The statement is trivial for $k=0$, since $G(0)=\mathbb I$. Then we can prove the proposition by induction. Assume that $G(k)$ is invertible. Then we rewrite the matrix recursion relation by multiplying both side with $C$ (by previous proposition, it's not important if on the left or on the right), then
	\be\label{eq:app1}
	\begin{split}
		T(k+1)&= \Big(1+\frac{\epsilon}{1+\epsilon k}\Big)T(k)- \frac{\epsilon}{1+\epsilon k} T(k) ^2=\\
		&=q_k ^{-1}T(k)\left[\mathbb I-p_k T(k)\right],
	\end{split}
	\ee
	with $T(k)=CG(k)$ (therefore with the initial condition $T(0)=C$). 	Taking the determinant of both sides and using the Binet theorem, we have
	\be
	\det{T(k+1)}=q_k ^{-p}\det{T(k)}\cdot \det{\left[\mathbb I-p_k T(k)\right]}.
	\ee
	But now
	\be
	\begin{split}
		\det{\left[\mathbb I-p_k T(k)\right]}&=\exp\log \det{\left[\mathbb I-p_k T(k)\right]}=\exp\text{Tr} \log{\left[\mathbb I-p_k T(k)\right]}.
	\end{split}
	\ee
	We can expand in series the logarithm of the matrix:
	\be
	 \log{\left[\mathbb I-p_k T(k)\right]}=-\sum_{n=1}^\infty \frac{p_k^n}{n}T^n(k),
	\ee
	which converges for $\lVert p_k T(k)\rVert=\lVert p_k CG(k)\rVert<1$ for all $k$ (note that this condition imposes the constraint for $\epsilon$ to be less than a critical value $\epsilon_c$, but we pospone this discussion). With the convergence ensured, it follows that
	\be
	\exp\text{Tr} \log{\left[\mathbb I-p_k T(k)\right]}=\exp\text{Tr}\Big(-\sum_{n=1}^\infty \frac{p_k^n}{n}T^n(k)\Big)>0.
	\ee
	Since all terms on the r.h.s of \eqref{eq:app1} are non-vanishing it follows that $\det{T(k+1)}=\det{C}\det G(k)\neq 0$, which implies that $\det{G(k)}\neq 0$ since $C$ is invertible. Then, also $G(k)$ is invertible for all $k$.
\end{proof}

\begin{Proposition}
	The matrices $T(k)$ and $G(k)$ are positive-definite for all $k$.
\end{Proposition}
\begin{proof}
	Again, the statement is trivial for $k=0$. Then, we will prove the proposition inductively. Suppose that all $G(l)$ are positive-definite for $l=0,\dots,k$. Since all $G$s commutes with $C$, then also $T(l)$ are positive-definite for $l=0,\dots,k$. Since $T(k)$ is invertible for each $k$, we can take the inverse of Eq. \eqref{eq:app1}:
	\be
	T^{-1}(k+1)=q_k [\mathbb I -p_k T(k)]^{-1}T(k)^{-1}.
	\ee
	But now
	\be
	[\mathbb I -p_k T(k)]^{-1}T(k)^{-1}=\sum_{n=0}^\infty p_k^n T(k)^n T(k)^{-1}=\sum_{n=0}^\infty p_k^n  T(k)^{n-1},
	\ee
	again converging for $\lVert p_k T(k)\rVert=\lVert p_k CG(k)\rVert<1$ for all $k$. In a more transparent form we have
	\be\label{eq:Trec}
	T^{-1}(k+1)=q_k T^{-1}(k)+q_k p_k\mathbb I +q_k\sum_{n=2}^\infty p_k^n  T(k)^{n-1}.
	\ee
	Under the hypotesis of convergence, $T^{-1}(k+1)$ is therefore a (infinite) sum of positive-definite matrices, then it is positive-definite by itself. Then, since the inverse of a positive-definite matrix is itself positive-definite, the same result holds for $T(k+1)$. But $G(k+1)= T(k+1)C^{-1}$, and since both $C^{-1}$ and $T(k+1)$ are positive-definite and commute (the proof is straightforward), then also the product $T(k+1)C^{-1}$ is positive-definite, inductively proving the proposition.
\end{proof}
At this point, we are ready to prove the
\begin{Lemma}
	For each $k$, there can be found a finite real number $c_k$ which is greater or equal to $\lVert CG(k)\rVert$. As a consequence, the sequence is bounded from above by $\bar c= \underset{k}{\text{max }} c_k$.
\end{Lemma}
\begin{proof}
Applying iteratively the recurrence relation \eqref{eq:Trec} and recalling that $T(0)=C$, it's easy to show that
\be\label{eq:Trec0}
T^{-1}(k)=\frac{C^{-1}}{1+\epsilon k}+N_{k-1}\mathbb I+R(k-1),
\ee
where the rest operator $R(k)$ is defined as
\be
R(k-1)= \sum_{l=0}^{k-1} \frac{1+\epsilon l}{1+\epsilon k}\sum_{n=2}^\infty p_l ^n T(l)^{n-1},
\ee
and
\be
N_k = \sum_{l=0}^k p_l \prod_{s=l}^k q_s.
\ee
The latter is an increasing function with $k$ with values in the range $[0,1]$ and such that $N_k \underset{k \rightarrow \infty}{\sim} 1$. Moreover, it's clear that $N_{-1}=0$. Since $T(k)$ is real and symmetric matrix, it can be diagonalized with eigenvalues $\tau_\mu (k)$ (which are positive since it is also positive-definite). Then, the spectrum of the inverse $T^{-1}(k)$ consists in the values $\sigma [T^{-1}(k)]= \{ \tau_{\mu}^{-1}(k)\vert \tau_{\mu}(k)\in \sigma [T(k)]\}$. Then, the minimum of the spectrum of $T^{-1}(k)$ is clearly $\lVert T(k)\rVert ^{-1}$. Therefore, taking the minimal eigenvalue of equation \eqref{eq:Trec0}, we have
\be
\lVert T(k)\rVert^{-1}= \text{min}\, \sigma\Big[\frac{C^{-1}}{1+\epsilon k}+N_{k-1}\mathbb I+R(k-1)\Big]\ge  \text{min}\, \sigma\Big[\frac{C^{-1}}{1+\epsilon k}+N_{k-1}\mathbb I\Big],
\ee
since also the rest operator is a positive-definite operator (and as a consequence, its contribution to the spectrum is positive). With the same reasoning, the quantity on the r.h.s. is nothing but
\be
\text{min}\, \sigma\Big[\frac{C^{-1}}{1+\epsilon k}+N_{k-1}\mathbb I\Big]=\frac{\lVert C\rVert ^{-1}}{1+\epsilon k}+N_{k-1}.
\ee
Then, by taking the inverse of the previous inequality, we get
\be
\lVert T(k)\rVert\le \Big(\frac{\lVert C\rVert ^{-1}}{1+\epsilon k}+N_{k-1}\Big)^{-1}= \frac{\lVert C \rVert}{\frac{1}{1+\epsilon k}+N_{k-1}\lVert C\rVert }=c_k.
\ee
This proves our assertion.
\end{proof}
A remark here is that the inequality is indeed an equality for $k=0$, since $T(k)=C$. This is important in the following
\begin{Corollary}\label{cor:ollario}
	The critical value of the unlearning strength $\epsilon_c$ is fixed by the norm of correlation matrix.
\end{Corollary}
\begin{proof}
	We recall that, in order to have a convergent algorithm, the unlearning strength $\epsilon$ has to satisfy the criterion $p_k \lVert T(k)\rVert <1$ for all $k$. By using the previous Lemma, we see that
	\be
	p_k \lVert T(k)\rVert\le  \frac{\epsilon \lVert C \rVert}{\frac{1+\epsilon(k+1)}{1+\epsilon k}+N_{k-1}[1+\epsilon(k+1)]\lVert C\rVert }.
	\ee
	It is important to notice that the denominator in this inequality is an increasing function of $k$. This means that, if the unlearning strength is chosen to have $p_0 \lVert T(0)\rVert <1$, then it would valid for all $k$. But, from our previous consideration (for $k=0$ it is an equality)
	\be
	p_0 \lVert T(0)\rVert=\frac{\epsilon \lVert C \rVert}{1+\epsilon}  <1,
	\ee
	meaning that the unlearning algorithm converges if and only if $\epsilon <\epsilon_c = \frac{1}{\lVert C \rVert -1}$.
\end{proof}
Once all of these results have been proved, we will finally prove that the unlearning algorithm converges (for $\epsilon < \epsilon _c$) to the desired solution $G(k)\rightarrow C^{-1}$. This will be proved by norm estimation in the large $k$ limit.
\begin{Theorem}[Convergence]
	The unlearning algorithm \eqref{eq:Gmatrix} converges to the stationary solution $G(\infty)= C^{-1}$ in the sense defined by the operator norm.
\end{Theorem}
\begin{proof}
	Let us start again with the equality
	\be
	T^{-1}(k)=\frac{C^{-1}}{1+\epsilon k}+N_{k-1}\mathbb I+R(k-1).
	\ee
	The first two terms on the r.h.s. have simple contributions in the large $k$ limit. In particular, the first one has a vanishing norm for $k\rightarrow \infty$, so it would not contribute to the final solution. We have now to evaluate the norm of the rest operator:
	\be
	\begin{split}
	\lVert R(k-1)\rVert&\le \sum_{l=0}^{k-1}\frac{1+\epsilon l}{1+\epsilon k}\Big\lVert \sum_{n=2}^\infty p_l ^n T(l)^{n-1}\Big\rVert \le  \sum_{l=0}^{k-1}\frac{1+\epsilon l}{1+\epsilon k} \sum_{n=2}^\infty p_l ^n \left\lVert T(l)\right\rVert ^{n-1}=\\
	&=\sum_{l=0}^{k-1}p_l\frac{1+\epsilon l}{1+\epsilon k}  \sum_{n=1}^\infty p_l ^n \left\lVert T(l)\right\rVert ^{n}=\frac{\epsilon}{1+\epsilon k} \sum_{l=0}^{k-1}\frac{1+\epsilon l}{1+\epsilon (l+1)}  \frac{p_l \left\lVert T(l)\right\rVert}{1-p_l\left\lVert T(l)\right\rVert},
	\end{split}
	\ee
	since $p_l \left\lVert T(l)\right\rVert<1$. To evaluate the last sum in this equation we adopt a counting argument by analyzing each factor. For the first one:
	\be
	\frac{1+\epsilon l}{1+\epsilon(l+1)}\sim \mathcal O (l^0),
	\ee
	for large enough $l$. For the second factor:
	\be
	\begin{split}
	&\frac{p_l \left\lVert T(l)\right\rVert}{1-p_l\left\lVert T(l)\right\rVert}\le\\&\le\epsilon \frac{1+\epsilon l}{1+\epsilon (l+1)}\frac{\lVert C \rVert}{1+\lVert C \rVert N_{l-1}(1+\epsilon l)} \Big(1-\epsilon \frac{1+\epsilon l}{1+\epsilon (l+1)}\frac{\lVert C \rVert}{1+\lVert C \rVert N_{l-1}(1+\epsilon l)}\Big)^{-1}=\\&=
	\frac{\epsilon(1+\epsilon l)\lVert C\rVert}{[1+\epsilon (l+1)][1+\lVert C \rVert N_{l-1}(1+\epsilon l)]}\sim \mathcal O (l^{-1}).
	\end{split}
	\ee
	Then, globally we have
	\be
	\frac{1+\epsilon l}{1+\epsilon(l+1)}\frac{p_l \left\lVert T(l)\right\rVert}{1-p_l\left\lVert T(l)\right\rVert}\sim \mathcal O (l^{-1}).
	\ee
	To evaluate the behavior for $k \rightarrow \infty$,\footnote{In this limit, we can replace $k-1$ in the sum simply with $k$.} we then take a large - but finite - integer $\bar l$ such that $1 \gg \bar l \gg k$ and split the sum as
	\be
	\begin{split}
	&\sum_{l=0}^{k}\frac{1+\epsilon l}{1+\epsilon (l+1)}  \frac{p_l \left\lVert T(l)\right\rVert}{1-p_l\left\lVert T(l)\right\rVert}\\&= \sum_{l=0}^{\bar l-1}\frac{1+\epsilon l}{1+\epsilon (l+1)}  \frac{p_l \left\lVert T(l)\right\rVert}{1-p_l\left\lVert T(l)\right\rVert}+\sum_{l=\bar l}^{k}\frac{1+\epsilon l}{1+\epsilon (l+1)}  \frac{p_l \left\lVert T(l)\right\rVert}{1-p_l\left\lVert T(l)\right\rVert}.
	\end{split}
	\ee
	Since $\bar l$ is large, the terms in the second sum are well-approximated with $l^{-1}$ (corrections are subleading in $l$), while the first sum is a finite number:
	\be
	\sum_{l=0}^{k}\frac{1+\epsilon l}{1+\epsilon (l+1)}  \frac{p_l \left\lVert T(l)\right\rVert}{1-p_l\left\lVert T(l)\right\rVert}\sim \text{ finite contributions }+H_k-H_{\bar l-1},
	\ee
	where $	H_s = \sum_{l=1}^s l^{-1}$ is the harmonic number. Since $\bar l$ is finite, the term $H_{\bar l-1}$ can be incorporated in the finite contributions, leaving only with
	\be
	\sum_{l=0}^{k}\frac{1+\epsilon l}{1+\epsilon (l+1)}  \frac{p_l \left\lVert T(l)\right\rVert}{1-p_l\left\lVert T(l)\right\rVert}\sim \text{ finite contributions }+H_k.
	\ee
	It is now well-known that the asymptotical behavior of $H_k$ for large $k$ is $H_k \sim \mathcal O (\log k)$. As a consequence, we find that the leading contribution goes as
	\be
	\frac{\epsilon}{1+\epsilon k} \sum_{l=0}^{k}\frac{1+\epsilon l}{1+\epsilon (l+1)}  \frac{p_l \left\lVert T(l)\right\rVert}{1-p_l\left\lVert T(l)\right\rVert}\sim \mathcal O (\log k /k)+ \text{ subleading contributions},
	\ee
	and therefore $\lVert R(k-1)\rVert$ vanishes in the $k\rightarrow \infty$ limit. By these norm estimation and since $N_k \underset{k\rightarrow \infty}\sim 1$, we can therefore conclude that
	\be
	T^{-1}(k)\underset{k\rightarrow \infty}{\sim} \mathbb I.
	\ee
	Recalling that $T(k)= CG(k)$, it immediately follows that $G(k)\rightarrow C^{-1}$ in the large $k$ limit, as claimed.
\end{proof}

\end{document}